\documentclass[10pt]{fairmeta}
\usepackage{graphicx}
\usepackage{subcaption}
\usepackage{wrapfig}
\usepackage{booktabs}
\usepackage{caption}
\usepackage{array}
\usepackage{algorithm}
\usepackage{algpseudocode}
\usepackage{amsmath,amssymb,booktabs}
\usepackage{amsthm}
\newtheorem{theorem}{Theorem}
\newtheorem{lemma}{Lemma}
\newtheorem{corollary}{Corollary}
\title{Action-Consequence Alignment for Reliable Planning and Self-Improving in Latent World Models}

\author[1,2]{Jinping Wang}
\author[2]{Zhiqiang Gao}
\author[3]{Xiantong Zhen}
\author[1,3]{Ling Shao}
\affiliation[1]{University of the Chinese Academy of Sciences}
\affiliation[2]{International Frontier Interdisciplinary Research Institute (IFIRI), Wenzhou-Kean University}
\affiliation[3]{OPI Labs}
\abstract{
Latent world models learn to predict observed transitions, yet low prediction error alone does not guarantee reliable planning. Inspired  by self-tickling experiments in neuroscience showing that disrupting motor–sensory correspondence increases prediction mismatch, we examine whether learned world models preserve an analogous action–consequence correspondence.
The results show nearby alternatives can receive lower prediction errors despite producing physical outcomes farther from the recorded target. With that future treated as a goal, this reveals a concrete prediction–planning mismatch: the model assigns a lower cost to an action that achieves the target less accurately. To mitigate this gap, we introduce  Action–Consequence Alignment (ACA), a training objective that complements forward prediction by penalizing the prediction-error advantage of locally searched alternatives over factual actions without additional model components or environment interactions during training. The same principle can also guide additional data collection for self-improvement. We demonstrate that across diverse environments and evaluation settings, ACA improves planning performance and reduces real goal error, while ACA-guided data collection outperforms random local sampling. These results support action–consequence alignment as a practical principle for bridging predictive learning and reliable planning.
}

\projectpage{https://github.com/smithgun2005/ACA}
\begin{document}
\maketitle
\section{Introduction}

World models \cite{worldmodels} provide an internal predictive mechanism for agents to anticipate future states and plan towards the desired outcomes. Instead of reconstructing high-dimensional observations, recent latent world models aim to predict future states and learn  compact representations directly in the latent space \cite{jepa,assran2025vJEPA-2,maes2026leworldmodel}. However, low prediction error on observed transitions does not by itself guarantee reliable comparisons between candidate actions and their consequences, as required for reliable planning. In many visual control environments, the next observation is highly predictable from  visual inertia, repetitive movement patterns and correlations induced by the behavior policy. A model can still achieve low prediction error while responding weakly to its action inputs \cite{zhang2026identifiability,li2026dueling,coco}.

Recent end-to-end latent world models have introduced anti-collapse objectives \cite{bardes2021vicreg,maes2026stable,inverse,zhang2026delta,balestriero2025lejepa} to learn stable predictive representations. However, preventing representation collapse does not ensure that a predictor captures the role of actions in producing future observations. 
Several recent approaches in latent action-conditioned world model recognize related failures. SMWM \cite{inverse} and Delta-JEPA \cite{zhang2026delta} introduce an inverse dynamic model to recover executed actions from the consecutive representations or their latent displacement. Other recent related approaches on different classes of world models also notice related problems and mainly introduce inductive biases or post-hoc verification mechanisms to inject action–consequence causality \cite{coco,li2026dueling,seo2026acid}. 
These concerns point to a mismatch between predictive training and planning objectives. Training fits observed futures under recorded actions, whereas planning optimizes actions to make predicted futures match a desired goal. Accurate predictions on recorded transitions therefore do not by themselves ensure that the resulting planning costs favor actions with better physical outcomes. We refer to this mismatch between factual prediction fitting and reliable action selection as the prediction–planning gap. 
This observation raises a fundamental question:
can the causal correspondence between an executed action and its observed consequence provide an intrinsic consistency principle for training that helps mitigate this gap?

%When the predictor is conditioned on actions, the latent world models provide a concise and effective paradigm for visual control that the agent is able to generate candidate action sequences in the learned latent space and select the sequence has the closest prediction result compared to the target. However, this approach is based on an fundamental premise: the latent action-conditioned should not only predict what might happen next, but also predict what will result from executing a given action.

We approach this question from the classical theory of forward models in neuroscience in which the sensorimotor control theory bears a strong similarity to planning in latent world model. An internal forward model predicts the sensory consequence of a motor command and update the model using the signal that discrepancies between the predicted and observed consequences provide \cite{miall1996forward,wolpert1995internal}. Self-tickling experiments show that: when the temporal or spatial relationship between a movement and its sensory consequence is perturbed, the resulting sensation becomes more surprising \cite{blakemore1999spatio}. 
Inspired by this experimental logic, we examine a trained world model by keeping the current state and observed future fixed, varying the conditioning action, and comparing the resulting prediction errors. 
We find nearby alternatives that receive lower errors than the factual action, even though executing them from the same simulator state moves the physical outcome away from the recorded target. With the recorded future treated as a goal, the model thus favors an action that achieves the target less accurately.

%This observation indicates a gap between prediction and planning in latent world models. During training, the model minimizes prediction error on observed transitions under recorded actions. During planning, the model is held fixed and candidate actions are optimized to minimize the predicted discrepancy from a desired outcome. 
%Fitting factual predictions therefore does not explicitly train the relative action comparisons on which planning relies, leaving room for alternative actions to receive misleadingly low costs.
These observations are consistent with the prediction–planning gap discussed above, suggesting that the actions favored by a learned predictor may not lead to better physical goal attainment.
To mitigate this prediction–planning gap, we introduce Action–Consequence Alignment (ACA). For each observed transition, ACA uses a single action-gradient step to search for a nearby alternative. ACA then penalizes a positive minibatch-averaged prediction-error
advantage of the searched alternatives over factual actions \cite{sam}. Combined with forward prediction, this objective provide direct supervision on how the predictor evaluates neighboring actions that may mislead planning, without additional model components or environment interaction during offline training.

%Cost-landscape visualizations and evaluations of executed plans further examine how training action comparisons affects downstream planning, showing lower real goal errors for plans selected by ACA-trained models
We demonstrate across all environments with multiple training and evaluation settings, ACA consistently achieves better planning performance. Moreover, ACA offers an effective technique to guide additional data collection for self-improvement. Our main contributions can be concluded as:
\begin{itemize}
% Inspired by self-tickling experiments in neuroscience, we probe latent world models through controlled action perturbations while keeping the current state and observed future fixed.
    \item We identify a prediction–planning mismatch in latent world models. Inspired by self-tickling experiments, we show that nearby counterfactual actions can receive lower prediction errors than the factual action while producing physical outcomes farther from the observed target.

    \item We introduce Action–Consequence Alignment (ACA), a simple training principle that searches for locally misleading counterfactual actions and penalizes their minibatch-averaged prediction-error advantage over the factual action, without additional model components or environment interaction during offline training. ACA improves planning performance across all evaluated environments and multiple training and evaluation settings.

    \item We further extend ACA into a self-improving learning mechanism. ACA enables the model to use its own action–consequence inconsistencies to identify informative counterfactual actions for additional data collection, forming a self-improvement loop that yields better subsequent planning than random local sampling.

\end{itemize}

\section{Problem Setting}
\paragraph{Offline transitions.}
Following \cite{inverse,maes2026leworldmodel}, we consider an offline dataset \(
\mathcal{D} = \{(x_t, a_t, x_{t+1})\}
\) of transitions, where \(x_t,x_{t+1} \in \mathcal{X}\) are consecutive observations and \(a_t \in \mathcal{A} \subseteq \mathbb{R}^{d_a}
\) is the action executed between them. We do not require a  behavior policy or a reward function. The goal is to learning a latent action-conditioned world model consisting of an encoder \(
f_{\theta} : \mathcal{X} \rightarrow \mathcal{Z}
\) and a forward predictor \(P_{\phi} : \mathcal{Z} \times \mathcal{A} \rightarrow \mathcal{Z},\), here \(\mathcal{Z}\) denotes the latent state space.

\paragraph{Latent forward prediction.}
Given a transition $(x_t, a_t, x_{t+1})$, we encode
both observations using the same encoder \(f_{\theta}\) and predict the next conditioned on an action $a$ with the forward predictor \(P_{\phi}\):
\begin{equation}
    z_t = f_{\theta}(x_t), \qquad
    z_{t+1} = f_{\theta}(x_{t+1}), \qquad
    \hat{z}_{t+1}(a) = P_{\phi}(z_t, a).
\end{equation}
The prediction objective is to match the embedding of the next observation with the following forward loss:
\begin{equation} \label{predloss}
    \mathcal{L}_{\mathrm{pred}}(\theta,\phi)
=
\mathbb{E}_{(x_t,a_t,x_{t+1})\sim\mathcal{D}}
\left[
\left\|
P_{\phi}(z_t,a_t)-z_{t+1}
\right\|_2^2
\right].
\end{equation}
This training objective alone will admits a trivial solution when the encoder maps every observation to a constant embedding $c$ (\( f_{\theta}(x_t)= f_{\theta}(x_{t+1})=c\)), then the prediction loss \(\mathcal{L}_{\mathrm{pred}}
=0\). To prevent such collapse, we use an anti-collapse regularization term \(\mathcal{L}_{anti}\) by adopting \(\mathcal{L}_{inv}\) \cite{inverse} and \(\mathcal{L}_{sig}\) \cite{maes2026leworldmodel}. 
\section{Action–Consequence Correspondence in Latent World Models}

\subsection{From Factual Prediction to Action Selection}
%%%%直接说因果问题

%%A central challenge in recent latent world models is learning stable representations to prevent representation collapse caused by trivial solution of the encoder mapping \cite{dinowm,maes2026leworldmodel,PLDM}. However, preventing representation collapse can not ensure the trained world model can captures the correct action consequence relation.

In visual control environments, a latent world model tends to rely on visual inertia, the statistical correlations with in the dataset to predict the future even if the model is formally conditioned on actions. Thus, the model may achieve a low prediction error while responding weakly to action variations \cite{liu2024multi,gu2024advancing, seo2026acid,coco}. More generally, for end-to-end world models, forward training and planning evaluate the predictor differently. Forward training measures prediction error under the action recorded in each transition, while planning compares candidate actions through their predicted consequences. The factual prediction objective does not explicitly compare the error assigned to the recorded action with that assigned to an alternative action for the same observed future. This raises a natural question: how can an observed action–consequence pair serve as a reference for evaluating alternative actions? A useful starting point comes from neuroscience, where perturbations to the correspondence between movement and sensory feedback reveal how internal forward models evaluate expected and unexpected consequences.

\subsection{Neuroscientific Inspiration: Action–Consequence Mismatch in Forward Models}

The classical forward model theory shows that the brain uses an efference copy of a motor command to predict the sensory consequence. The discrepancy between the predicted sensory consequence and the actual feedback is used to correct state estimates, modulate sensory responses and update the internal model \cite{wolpert1995internal,miall1996forward,bubic2010prediction}.
This structure resembles action-conditioned prediction in contemporary world models, which have also been widely discussed \cite{worldmodels,taniguchi2023worldpred,rao2024sensory,jepa}.

\begin{figure}[t
]
    \centering
    \includegraphics[width=1\linewidth]{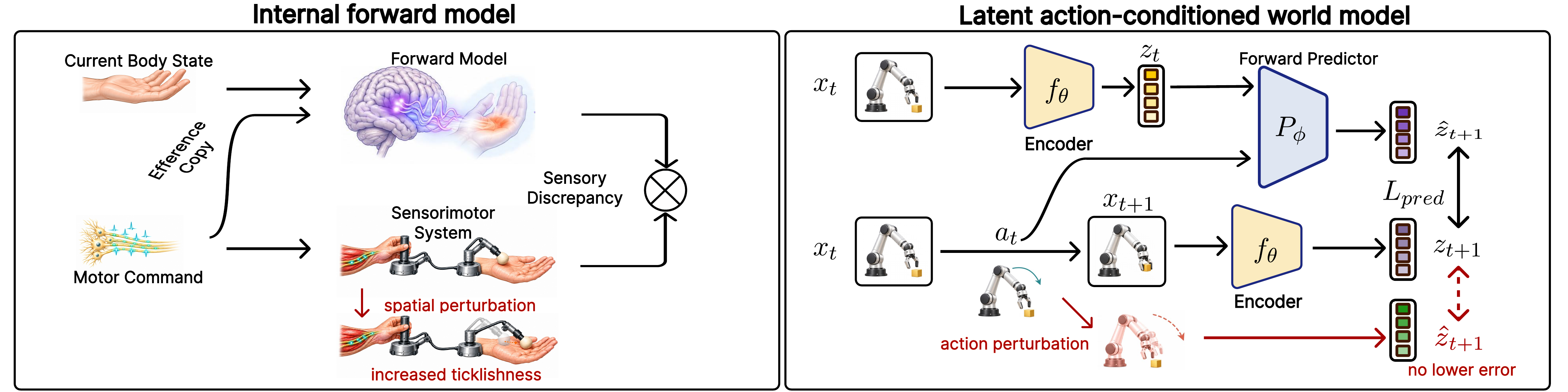}
    \caption{Structural correspondence between biological forward models (left) and latent action-conditioned world models (right). Both use an action signal to predict the resulting consequence and use the discrepancy between predicted and observed outcomes to update the internal model. 
}
    \label{fig:placeholder}
\end{figure}
This connection motivates us to investigate the relationship between action and forward prediction in latent world models. The self-tickling experiment offers a particularly revealing example. 
\cite{blakemore1999spatio} examined the action-sensation matching using a two robot apparatus which the participants moved one robot with one hand while a second robot delivered tactile stimulation to the other hand. The experiments show that when introducing a temporal delay between the movement and its tactile consequence or perturbing the spatial corresponding between their trajectories, the ticklishness increase progressively. 

%The experiment gives us an inspiration of examining whether forward prediction makes effective use of the action condition: keep the observed outcome fixed while perturbing the action used, if the forward prediction captures the distinct consequences associated with perturbed actions, a mismatched action condition should not produce a smaller prediction error—analogous to the increased ticklishness in the experiment.  
This perturbation based experimental logic inspires a direct probe of learned action comparisons. We keep the current state and the observed future fixed, perturbing the conditioning action and examine how the prediction error changes. Executing the corresponding actions then allows us to relate these model preferences to physical outcomes.

\begin{wrapfigure}{r}{0.5\textwidth}
    \centering

    \begin{subfigure}[t]{0.47\linewidth}
        \centering
        \includegraphics[width=\linewidth]{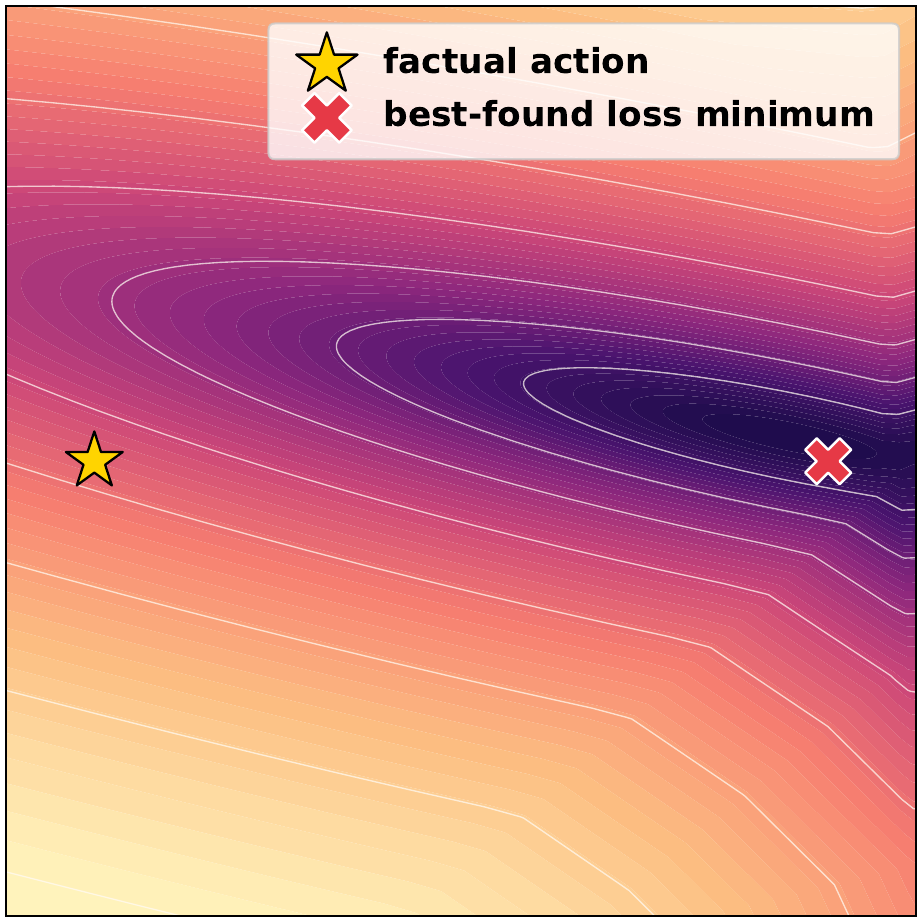}
        \caption{}
        \label{fig:a}
    \end{subfigure}%
    \hspace{0.01\linewidth}
    \begin{subfigure}[t]{0.47\linewidth}
        \centering
        \includegraphics[width=\linewidth]{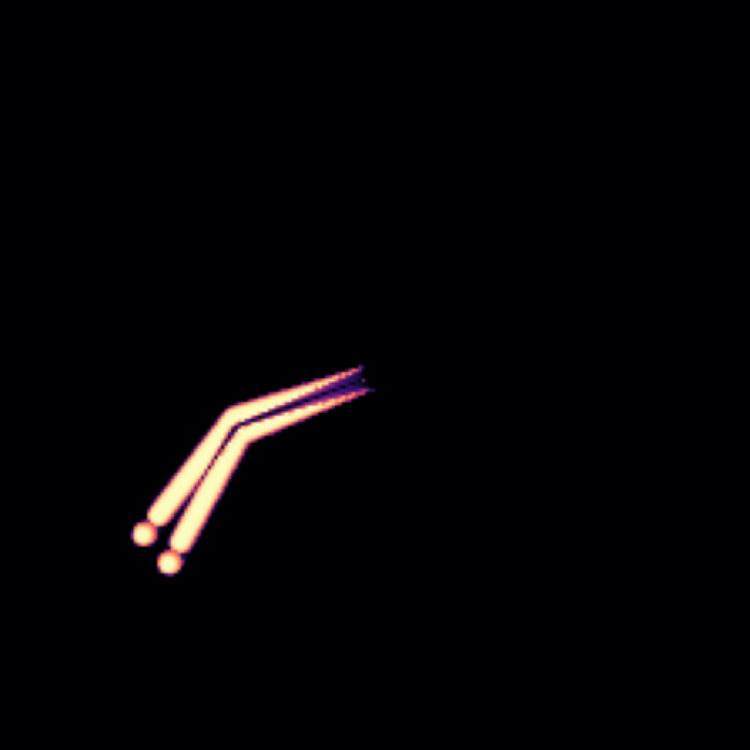}
        \caption{}
        \label{fig:b}
    \end{subfigure}
    \caption{(a) Local action-loss landscape for a fixed Reacher transition. The landscape is a two-dimensional slice of the 10-dimensional five-step action space. (b) Absolute pixel-wise difference between the MuJoCo renders obtained by executing the two action sequences for five steps from the same recorded initial simulator state. }
    \label{landscape}
\end{wrapfigure}
As illustrated in Fig.~\ref{landscape}, for a fixed transition, we visualize the action-loss landscape using the trained predictor and  search an alternative counterfactual action with lower prediction error. The resulting action attains a substantially lower prediction error than the factual action. However, executing the two actions from the same initial state produces visibly different physical outcomes. We make the same observations in more thorough experiments across four environments as shown in Appendix~\ref{app:ordering_reversals}.
%extends this example  and both baseline models. We measure the frequency of lower-error alternatives over 10,000 sampled transitions per environment and evaluate their physical consequences on 1,000 transitions with lower-error alternatives per model–environment–radius setting. When factual replay reproduces the recorded target, these alternatives receive lower model error despite moving away from that target under execution. 

These results expose a mismatch between the actions favored by the learned predictor and those that achieve the desired physical outcome. When the recorded future is treated as the goal, a nearby alternative can receive a lower prediction error than the factual action yet lead to a larger physical goal error. Thus, minimizing the learned prediction error can favor actions that move the system farther from the goal, which is a specific manifestation of the prediction-planning gap. This lower error under a mismatched action runs counter to the intuition drawn from the self-tickling experiments, where disrupting action–consequence correspondence makes sensory feedback less predictable. The contrast motivates a training principle that encourages the predictor to preserve this correspondence when comparing nearby actions, beyond fitting factual transitions alone.

\section{Action–Consequence Alignment }
We introduce the computational principle, Action–Consequence Alignment (ACA), which is used as a training objective to mitigate the prediction–planning mismatch for reliable planning. 
%ACA uses each observed transition as a factual reference, searches for nearby actions with lower prediction error, and directly penalizes their positive error advantage over the recorded action. This converts the comparison discussed in the previous section into a direct training principle.
\subsection{Factual--Counterfactual Prediction Comparison}
\label{sec:aca_comparison}

Consider an observed transition
$\tau=(x_t,a_t,x_{t+1})\sim\mathcal{D}$.
Keeping the current latent state $z_t$ and observed future
$z_{t+1}$ fixed, we define the prediction error under a
candidate action $a$ as
\begin{equation}
    E_\phi(a;\tau)
    =
    \left\|P_\phi(z_t,a)-z_{t+1}\right\|_2^2.
    \label{eq:aca_explanatory_error}
\end{equation}
ACA aggregates the prediction-error differences between factual and searched actions within each minibatch and penalizes a positive average advantage of the searched alternatives.Equal errors are permitted, allowing different actions to explain the same outcome.

This comparison has an ideal-model reference. If the conditioning variables determine the next latent state, an ideal predictor assigns zero error to the factual action, which is consequently a minimizer. Appendix~\ref{theory} relates expected prediction-error differences to differences between latent conditional means and clarifies the effect of conditional variability on individual transitions. ACA uses these factual–alternative comparisons as a local training surrogate motivated by the observed prediction–planning mismatch. The following subsection describes how we construct the alternative action efficiently.

\subsection{Searching for Alternative Actions}
\label{sec:aca_action_search}

Searching the full continuous action space for a challenging
alternative can be computationally expensive.
We therefore adapt the local perturbation framework of the traditional
Sharpness-Aware Minimization \cite{sam,li2024friendly,wang2026fix} to action-conditioned
prediction.
% Add the original SAM citation here.
SAM searches for loss increasing parameter perturbations
within a bounded neighborhood to improve the generalization ability of a model. We repurpose this framework for perturbing conditioning actions.
Starting from the factual action, ACA searches for a nearby alternative that decreases the prediction error for the same observed future.
The search identifies alternatives to which the current predictor assigns an error advantage, providing confusable candidates for the factual–alternative comparison.

Specifically, we define an action neighborhood centered at
the factual action:
\begin{equation}
    \mathcal{C}_\rho(a_t)
    =
    \left\{
        a\in\mathcal{A}:
        \|a-a_t\|_2\leq\rho
    \right\},
    \label{eq:aca_neighborhood}
\end{equation}
where $\rho>0$ controls the perturbation radius.
Within this neighborhood, the local search target is
\begin{equation}
    a_t^\star
    \in
    \arg\min_{a\in\mathcal{C}_\rho(a_t)}
    E_\phi(a;\tau).
    \label{eq:aca_inner_search}
\end{equation}

Even within this neighborhood, accurately solving the inner
problem can be difficult because the prediction error is
potentially nonconvex with respect to the action.
Repeating iterative optimization for every training transition
would also introduce substantial computational overhead.
Following the first-order inner approximation used in SAM,
we linearize the objective around the factual action and
construct the candidate using a single normalized
action-gradient step.

\begin{wrapfigure}{r}{0.4\textwidth}
    \centering
    \includegraphics[width=\linewidth]{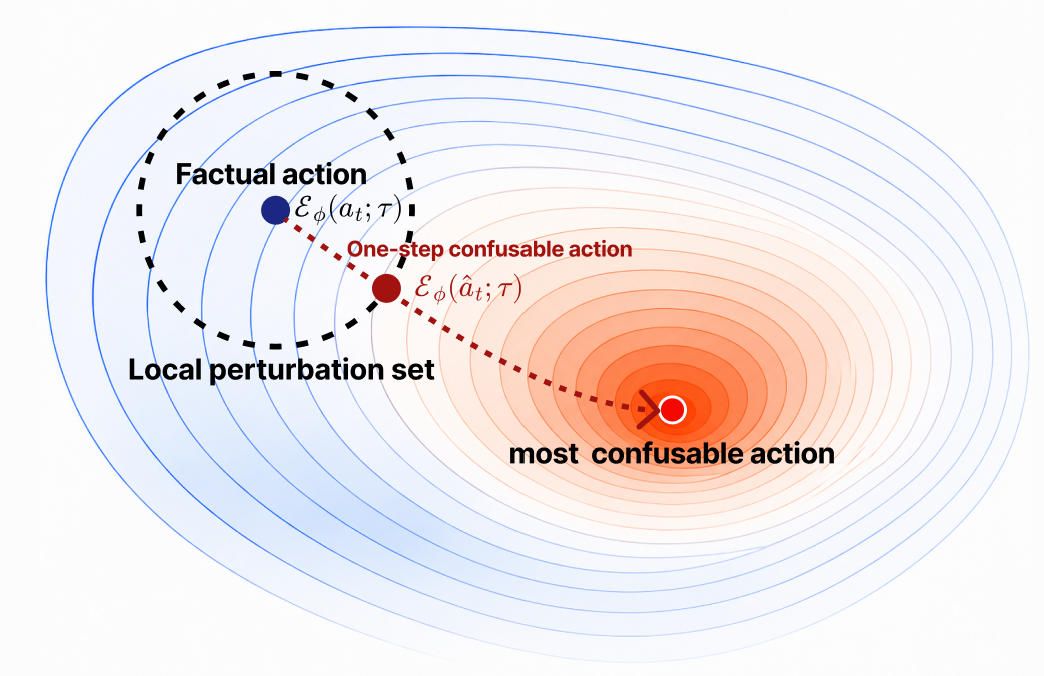}
    \caption{Illstration of the Action Consequence Alignment (ACA) principle.}
    \label{fig:landscape}
    \vspace{-2mm}
\end{wrapfigure}

Let
\(
    g_t
    =
    \left.
        \nabla_a E_\phi(a;\tau)
    \right|_{a=a_t}
    \label{eq:aca_action_gradient}
\)
denote the action gradient. 
We construct the searched alternative as
\begin{equation}
    \hat{a}_t
    =
    \operatorname{sg}
    \left[
        \Pi_{\mathcal{A}}
        \left(
            a_t-\rho
            \frac{g_t}{\|g_t\|_2+\epsilon}
        \right)
    \right],
    \label{eq:aca_searched_action}
\end{equation}
where $\Pi_{\mathcal{A}}$ projects the perturbed action onto
the valid action space and $\epsilon>0$ ensures numerical
stability.
The stop-gradient operator prevents differentiation through
the construction of $\hat{a}_t$, avoiding second-order
derivatives through $g_t$.
The prediction error $E_\phi(\hat{a}_t;\tau)$ remains
differentiable with respect to the encoder and predictor
parameters. This approximation requires a single gradient
computation with respect to the action, without an iterative
inner loop or additional environment interaction.
The resulting candidate is used in the ACA comparison
objective described next.

\subsection{ACA Objective}
\label{sec:aca_objective}

Given the searched alternatives, ACA compares their prediction errors with those of the factual actions within each minibatch. For a minibatch $\mathcal B$, we apply a hinge penalty to the
average factual-minus-alternative prediction-error difference:
\begin{equation}
\mathcal L_{\mathrm{ACA}}(\theta,\phi)
=
\mathbb E_{\mathcal B}
\left[
\left[
\frac{1}{|\mathcal B|}
\sum_{\tau\in\mathcal B}
\left(
E_\phi(a_t;\tau)-E_\phi(\hat a_t;\tau)
\right)
\right]_+
\right].
\end{equation}
The minibatch penalty is positive when the searched alternatives have a lower average prediction error than the factual actions.
We combine this term with the forward-prediction and anti-collapse objectives:
\begin{equation}
    \mathcal{L}(\theta,\phi)
    =
    \mathcal{L}_{\mathrm{pred}}(\theta,\phi)
    +
    \lambda_{\mathrm{anti}}\mathcal{L}_{\mathrm{anti}}
    +
    \lambda_{\mathrm{ACA}}
    \mathcal{L}_{\mathrm{ACA}}(\theta,\phi).
    \label{eq:aca_total_loss}
\end{equation}
The complete training procedure is summarized in Appendix~\ref{algorithm}.

The joint objective supervises both factual predictions and local relative errors. Forward prediction fits the observed consequence under the recorded action, while ACA penalizes
a positive minibatch-averaged error advantage of the searched alternatives. In this way, the same transition provides both a prediction target and a reference for action comparison. The resulting objective mitigates the prediction–planning gap by improving the usefulness of the learned predictor for goal-directed action selection.

\begin{table}[t]
    \centering
    \caption{Planning performance comparison on different environments. Success rates (\%) are reported as mean. The standard deviation is provided in Table~\ref{tab:planner_seed_sd}.
}
    \label{latentplanning}
    \renewcommand{\arraystretch}{1}
    \setlength{\tabcolsep}{12pt}
    
    \resizebox{1\linewidth}{!}{%
    \begin{tabular}{lcccccc}
        \toprule
        \textbf{Method}
        & \textbf{TwoRoom}
        & \textbf{TwoRoom*}
        & \textbf{Reacher}
        & \textbf{Push-T}
        & \textbf{\shortstack{OGBench-\\Cube}}
        & \textbf{\shortstack{OGBench-\\Cube*}} \\
        \midrule
        Random       & 30.0 & 2.2  & 14.0 & 3.0  & 43.0 & 5.4  \\
        Forward-only & 37.0 & 2.0  & 11.0 & 2.0  & 44.0 & 4.4  \\
        LeWM         & 94.0 & 73.8 & 67.0 & 81.0 & 59.0 & 30.8 \\
        SMWM         & 97.8 & 82.0 & 66.0 & 81.4 & 82.2 & 66.2 \\
        \midrule
        \textbf{LeWM+ACA}
        & $\mathbf{94.2}$
        & $\mathbf{79.4}$
        & $\mathbf{72.0}$
        & $\mathbf{83.2}$
        & $\mathbf{65.0}$
        & $\mathbf{34.0}$ \\
        \textbf{SMWM+ACA}
        & $\mathbf{98.8}$
        & $\mathbf{95.8}$
        & $\mathbf{71.4}$
        & $\mathbf{84.4}$
        & $\mathbf{85.6}$
        & $\mathbf{72.4}$ \\
        \bottomrule
    \end{tabular}%
    }
\end{table}

\section{Reliable Planning} \label{experiment}
% To evaluate the effectiveness of our ACA algorithm, we conduct experiments on actionable downstream control tasks.
% Our evaluation examines whether ACA improves goal-conditioned planning when added to end-to-end latent world models under a shared experimental protocol.
\subsection{Experiment Setup} 
\textbf{Environments and datasets.} 
We conduct experiment on four environments and offline datasets: TwoRoom \cite{PLDM}, Reacher \cite{Reacher}, Push-T \cite{pusht}, and OGBench-Cube \cite{park2025ogbench}, ranging from 2D navigation to articulated control and contact-rich manipulation. Noting that even random policies achieve relatively high success rates on OGBench-Cube and TwoRoom, which result in difficulties distinguishing the control capabilities of different methods, we adopt the Strict cube and TwoRoom protocol (OGBench-Cube*, TwoRoom*) in CLEAR-LeWM \cite{sun2026intactisomorphicintenttoactionlearning}. More details of the experiment datasets are given in Appendix~\ref{data}. 

\textbf{Architecture.} 
We adopt the same encoder and forward-prediction architecture following \cite{maes2026leworldmodel,inverse}. The encoder is a ViT tiny \cite{vit,vits} with around 5M parameters. The forward predictor \(P_\phi\) is a small transformer \cite{vaswani2017attention} with the action \(a_t\) injected via AdaLN-zero \cite{DIT}. 

\textbf{Training and planning protocol.}
To ensure fair comparison, we strictly follow the training an planning protocol of \cite{inverse}. All models are trained end-to-end using two-frame snippets and a single-step forward-prediction target. The datasets are loaded with a frame skip of five which the primitive actions within the skipped interval are concatenated to form a single model-level action. Every model is trained for 10 epochs, detailed hyper-parameters for training is shown in Appendix~\ref{experiment details}.  We evaluate the learned models through goal-conditioned trajectory optimization in the latent space. In the beginning of each planning episode, the agent will receive a current observation \(x_t\) and a goal observation \(x_g\). The frozen encoder will map them to \(z_c = f_\theta(x_t),  z_g = f_\theta(x_g).\) For each candidate action sequence,the frozen predictor is rolled out auto-regressively where:
\[
\hat{z}_{t+h+1}
=
P_\phi\!\left(\hat{z}_{t+h}, a_{t+h}\right),
\quad
\hat{z}_t = z_t.
\]
The candidate sequence is then evaluated using its terminal latent goal matching cost:
\[
\mathcal{C}(a_{t:t+H-1})
=
\left\|
\hat{z}_{t+H} - z_g
\right\|_2^2.
\]
The cost is optimized using the Cross-Entropy Method \cite{CEM}. 
All methods are evaluated with a fixed budget of 50 environment steps per episode to reach the goal. We evaluate each fixed trained checkpoint using five consecutive planner random seeds and report the mean success rate with standard deviation (SD). More detailed  experiment settings are in Appendix~\ref{experiment details}.

\subsection{Planning Performance}
\textbf{Baseline.} We compare ACA against various baselines under same experiment settings. Random uniformly samples action from the action space without 
using a learned world model; Foward-Only trains the encoder and the predictor with only the forward-prediction objective; LEWM \cite{maes2026leworldmodel} applies a Gaussian-matching regularization term to prevent representation collapse; SMWM \cite{inverse} uses an inverse dynamic model to prevent model collapse.   

\textbf{Results.} As shown in Table~\ref{latentplanning}, ACA improves both LeWorldModel and SMWM across the evaluated
environments and protocols. The improvements over SMWM indicate
that inverse-dynamics supervision leaves room for improving
downstream action selection through the forward predictor.
Under the shared planning protocol, adding ACA yields better
planning performance, supporting the value of explicitly training factual-counterfactual comparisons alongside factual prediction.

\begin{figure}[t]
    \centering
    \includegraphics[width=0.95\linewidth]{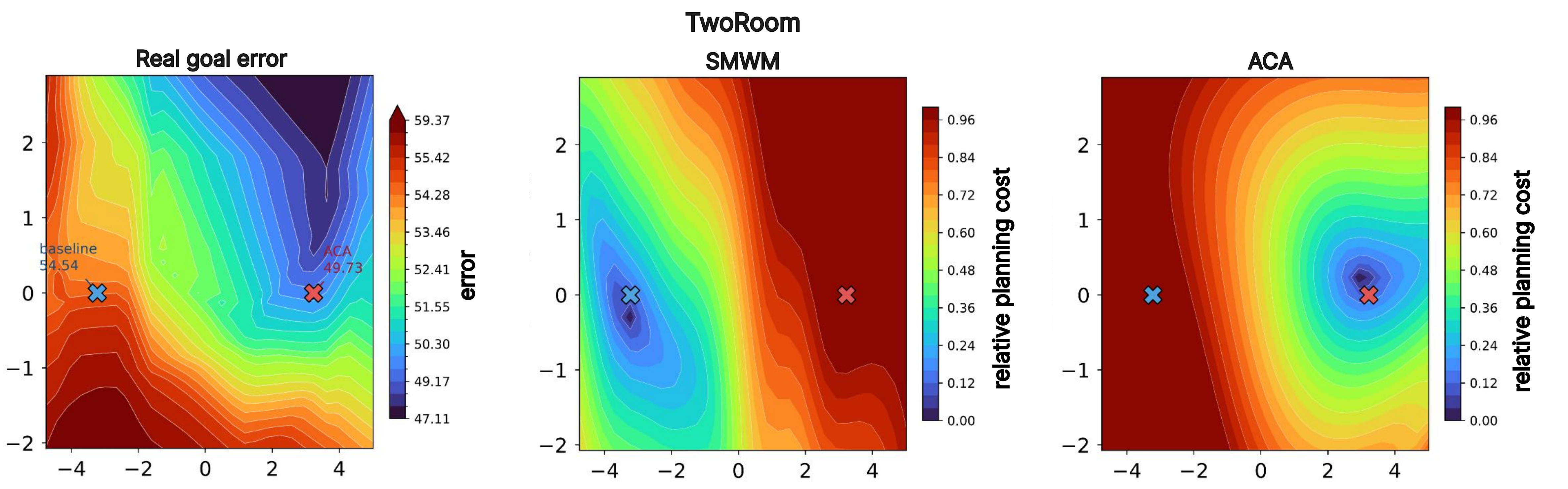}
    \caption{Visualization of planning-cost alignment on TwoRoom, more visualization results is shown in Appendix~\ref{app:visualizations}.}
    \label{cost_alignment}
\end{figure}

\subsection{Planning-Cost Alignment Analysis}

%Our diagnostic analysis identifies a mechanism by which explanatory inconsistency can mislead planning: the same predictor assigns prediction errors during training comparisons and goal-matching costs duringaction selection. We now examine whether training with ACA improves the resulting correspondence between model-based planning cost and real goal error. Given the trained baseline and ACA models, we run CEM under the same planning protocol and construct a shared two-dimensional slice over complete action plans. Each plan is evaluated using simulator rollouts and the planning costs of both models. Figure~\ref{cost_alignment} compares the real and model-based cost landscapes and shows that the ACA-selected plan achieves lower real goal error in the illustrated example. Figure~\ref{plotscore} further reports lower real goal errors for plans selected using ACA across the evaluated environments and base models. Together, these results provide empirical support for the proposed training--planning connection: incorporating action--consequence consistency into training improves the real outcomes of plans selected through the learned cost. Additional visualizations are provided in Appendix~\ref{experiment details}.
\begin{figure}[t]
    \centering
    \includegraphics[width=1\linewidth]{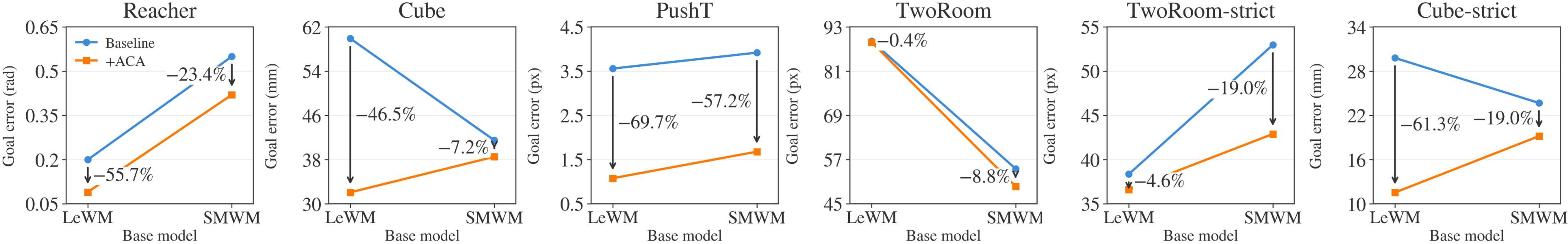}
    \caption{ACA substantially reduces the real planning error across environments, indicating improved alignment between theplanning cost and the actual consequences of executed actions.}
    \label{plotscore}
\end{figure}

The action perturbation experiment shows that lower prediction error can favor an alternative action with a different physical consequence. We now examine whether training with ACA improves the outcomes of actions selected through the trained models.
For the trained baseline and ACA models, we run CEM under the
same planning protocol and construct a shared two-dimensional
slice over complete action plans. The same plans are evaluated
using model-based planning costs and simulator rollouts.
Figure~\ref{cost_alignment} illustrates the discrepancy between
a model preference and the physical outcome of its selected
plan. The baseline favors a plan with low model cost even though its actual outcome has a larger goal error than the ACA-selected plan. The landscape visualization thus shows how optimizing a learned prediction-based objective can favor a physically inferior solution. In the illustrated example, training with ACA shifts the model preference toward a plan that achieves a better physical outcome.

Figure~\ref{plotscore} extends this evaluation across environments and both base models, showing lower real goal errors for plans selected using ACA-trained predictors. Together with the planning success rates, these results provide empirical support that local factual–counterfactual comparisons help mitigate the prediction–planning gap by improving the physical outcomes of selected plans. Additional visualizations are provided in Appendix~\ref{app:visualizations}.

 \subsection{Planning with limited training data}

We next examine whether ACA remains useful when training experience is limited, using only $5\%$ of the original training data. Restricted state–action coverage provides a challenging setting for transferring factual prediction learning to the candidate actions considered during planning. This experiment evaluates whether the local comparisons introduced by ACA improve downstream action selection under the same limited data budget.

We train models on Reacher for 20 epochs, Cube for
10 epochs, TwoRoom for 20 epochs and Push-T for 40 epochs to ensure convergence. We use cosine learning-rate annealing with an
initial learning rate of $10^{-4}$.
As shown in Figure~\ref{fewdata}, ACA outperforms LeWorldModel and SMWM
across the evaluated settings and reduces the performance
difference from their corresponding full-data baselines.
These results support the value of complementing factual prediction with action–consequence comparisons: with the same limited training data, ACA-trained models achieve better goal-directed planning on held-out tasks.

\begin{figure}[t]
    \centering
    \includegraphics[width=1\linewidth]{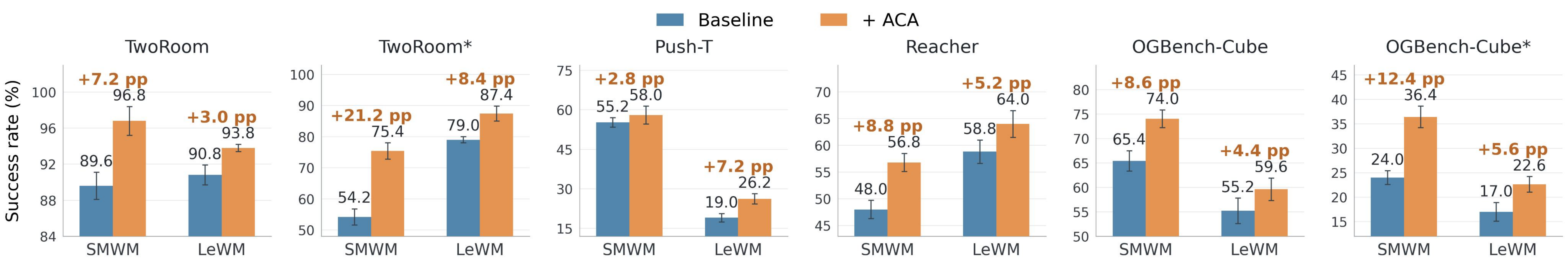}
    \caption{ACA achieves reliable planning under the limited data settings.}
    \label{fewdata}
\end{figure}

\subsection{Additional experiment}
 We provide hyperparameter sensitivity analysis (Appendix~\ref{experiment details}), additional planning-cost visualizations (Appendix~\ref{app:visualizations}), single-plan and long-horizon evaluations (Appendix~\ref{sec:additional_planning}), ablations of action search and regularization (Appendix~\ref{app:aca_ablation}), and variability across planner seeds (Appendix~\ref{appendixsd}).
\section{Self-Improving via the ACA principle}
%%%%todo: 引用和cube的结果

%Collecting action-labeled environment interactions often needs expensive computation with large time costs, thus acquiring informative transitions under a limited sampling budget remains an important problem in world model learning. 

\begin{wrapfigure}{r}{0.4\textwidth}
    \centering
    \includegraphics[width=\linewidth]{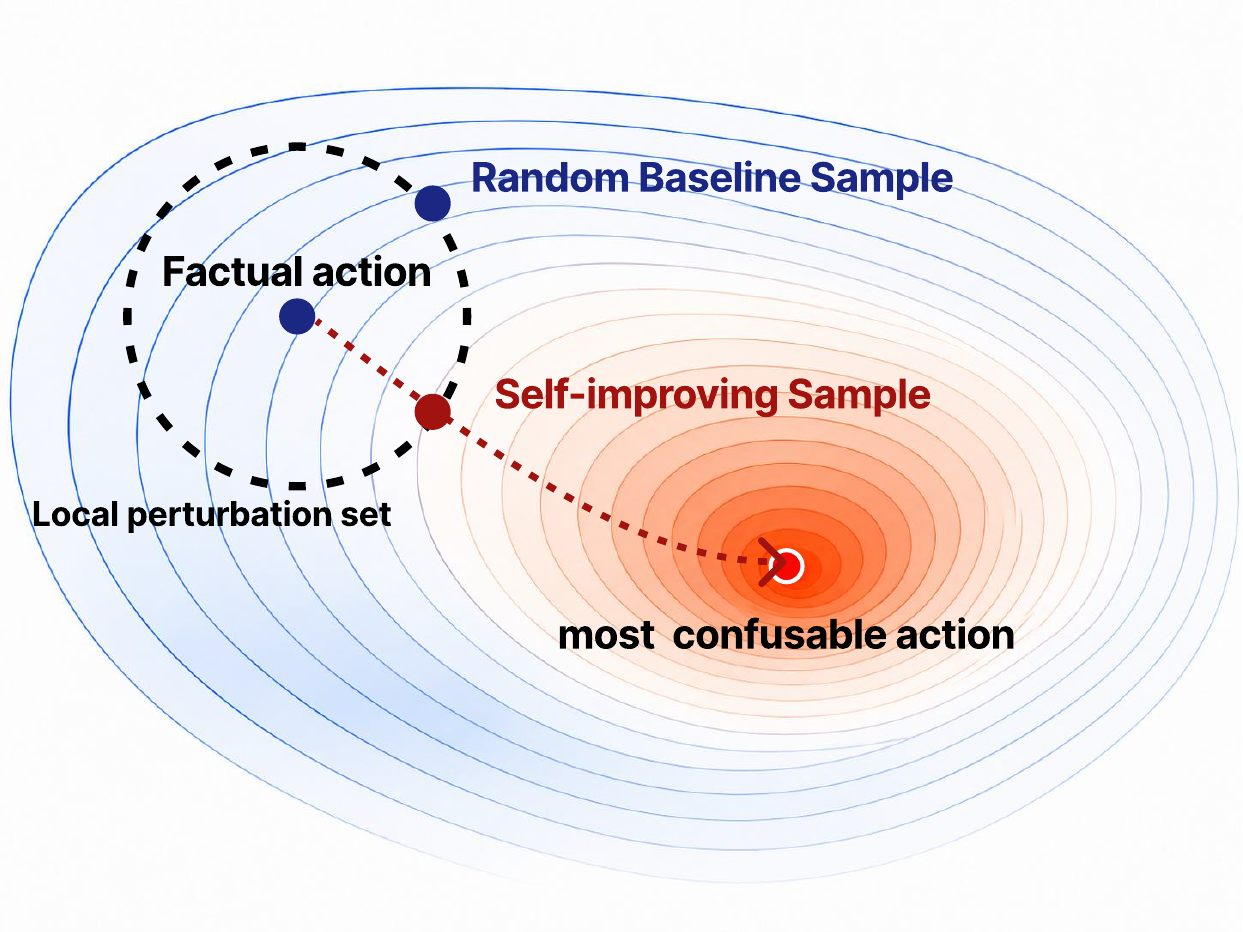}
    \caption{Overview of ACA self improving sampling method and random sampling baseline.}
    \label{fig:landscape}
    
\end{wrapfigure}
Identifying potential model failures can guide targeted data collection during training and support iterative self-improvement \cite{WAV}.
The ACA principle provides such a signal  by identifying nearby actions through searching for nearby alternatives with lower prediction error against a recorded future. We then execute these counterfactual actions from the corresponding simulator states, collect their actual outcomes, and augment the dataset for further training. This creates a self-improvement loop in which the local action preferences of the model guide the acquisition of additional experience.

We evaluate this procedure, starting from a model pretrained on $5\%$ of the original training data on Reacher and Cube datasets. We compare two sampling strategies using exactly the same reference-transition indices: the random baseline reuses the indices selected for self-improving sampling strategy. For each shared transition
$\tau=(x_t,a_t,x_{t+1})$, Self-improving sampling strategy follows the normalized negative gradient of the prediction error:
\begin{equation}
    a_t^{\mathrm{Self}}
    =
    \Pi_{\mathcal{A}}
    \left(
        a_t-\rho\frac{g_t}{\|g_t\|_2+\epsilon}
    \right),
    \quad
    g_t=
    \left.
    \nabla_a E_{\phi}(a;\tau)
    \right|_{a=a_t}.
    \label{eq:self_sampling}
\end{equation}
The random baseline instead perturbs the same factual action along a random unit direction $u_t$, using the same perturbation scale and action-space projection:
\begin{equation}
    a_t^{\mathrm{Random}}
    =
    \Pi_{\mathcal{A}}\left(a_t+\rho u_t\right),
    \quad
    \|u_t\|_2=1,
    \label{eq:random_sampling}
\end{equation}
where $u_t$ is sampled without using the prediction-error
gradient. For both strategies, the selected action is executed from the same reference simulator state, and its actual resulting observation is collected to construct a new training transition. For all models and datasets, we use a constant learning rate of \(10^{-4}\). Both base models are trained for 25 epochs on Reacher and 30 epochs on OGBench-Cube to ensure convergence. For both sampling strategies, additional transitions are collected at epochs 5 and 10, with each sampling round containing a number of samples equal to 10\% of the initial training subset. To ensure a fair comparison, the two strategies use exactly the same indices of initial simulator states and differ only in the perturbation directions applied to the corresponding factual actions.

\iffalse
\begin{wrapfigure}{r}{0.45\textwidth}
    \centering
    \includegraphics[width=\linewidth]{figure/self-improving-result.pdf}
    \caption{Self-improving results}
    \label{selfreacher}
\end{wrapfigure}
\fi

As shown in Fig.~\ref{selfreacher}, ACA-guided sampling consistently achieves higher planning performance than random local sampling, and also surpasses the reference model trained without additional data. Because both sampling strategies use the same initial-state indices, sampling budgets, and perturbation scale, their difference isolates the effect of the selected perturbation direction. These results suggest that executing alternatives favored by the current predictor can provide useful additional action–consequence supervision. Under the evaluated sampling protocol, this directed collection improves subsequent planning more than random local perturbations.

%More experiment results on cube and details is shown in Appendix~\ref{experiment details}. 

\section{Related work} \label{related work}

\textbf{Latent world models} \quad
World models learn compact dynamics for prediction and control. Early work studied control through learned representations \cite{worldmodels}, latent-space planning \cite{hafner2019learning}, and imagined trajectories \cite{hafner2019dream}. Predictive representation learning further reduces reliance on reconstruction: CPC \cite{oord2018representation} uses contrastive future prediction, while I-JEPA \cite{ijepa} and V-JEPA \cite{vjepa} predict target embeddings directly. DINO-WM \cite{dinowm}, V-JEPA 2 \cite{assran2025vJEPA-2}, and PLDM \cite{PLDM} extend latent prediction to planning and control. For end-to-end learning, VICReg \cite{bardes2021vicreg} and LeJEPA \cite{balestriero2025lejepa} introduce anti-collapse regularization, while LeWorldModel combines such regularization with latent dynamics learning from pixels \cite{maes2026leworldmodel}. 
Garrido et al.~\cite{garrido2026learning} learn continuous latent actions from in-the-wild videos and connect known controls to the learned action space for planning.
CoLA-World~\cite{wang2025co} enables joint learning of latent action and world models through a representation warm-up that mitigates collapse.
DiLA~\cite{zhang2026dila} separates dynamics-relevant structure from visual content and models latent-action-conditioned transitions in structure space.
CLAW~\cite{ayalew2026claw} jointly learns continuous latent actions and a world model from action-free videos using adversarial latent regularization.
%Our work instead focuses on the relative prediction errors assigned to factual and alternative actions.

\textbf{Action--consequence dependence} \quad
Internal forward models predict sensory consequences from motor commands and states \cite{wolpert1995internal,miall1996forward}, while self-produced touch experiments highlight movement--sensation correspondence \cite{blakemore1999spatio}. Forward and inverse dynamics have also been used for representation learning \cite{agrawal2016learning,pathak2017curiosity}. Recent latent world models strengthen action dependence through inverse dynamics in SMWM \cite{inverse}, latent difference decoding in Delta-JEPA \cite{zhang2026delta}, action-conditioned consistency in CoCo \cite{coco}, action-dependent decomposition in Dueling World Models \cite{li2026dueling}, and inverse-dynamics-aware planning costs in ACID \cite{seo2026acid}. Related theory studies identifiability through conditional action variation \cite{zhang2026identifiability}. Wang et al.~\cite{wang2026predicting} train a latent compatibility predictor using counterfactual action sequences for goal-directed navigation. %ACA is complementary: it directly penalizes cases where a nearby alternative action explains the factual future better than the factual action, without additional model components or predefined action transformations.

\begin{figure}[t]
    \centering
    \includegraphics[width=1\linewidth]{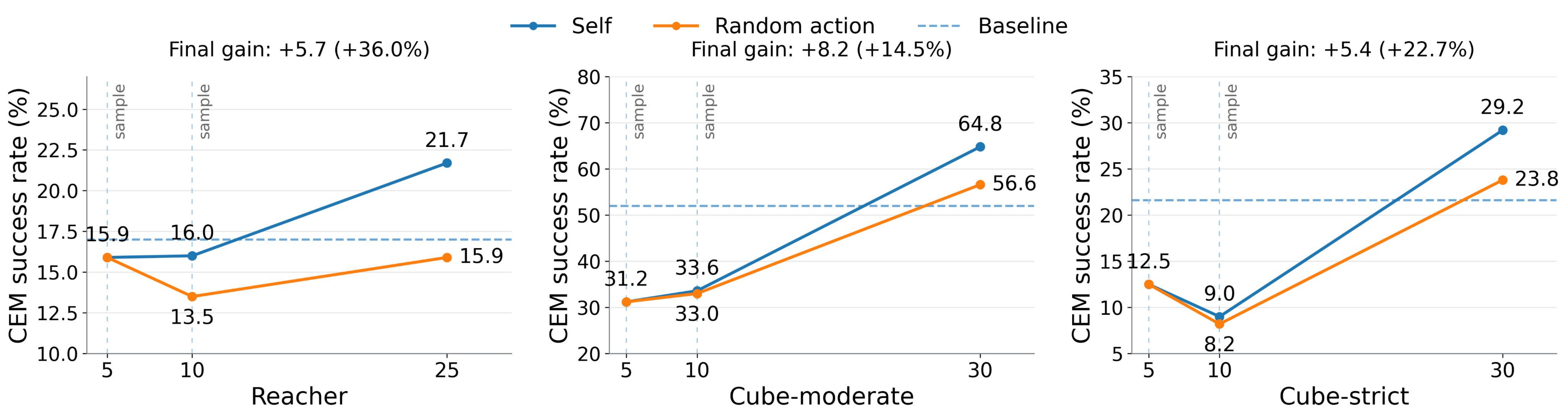}
    \caption{ACA sampling enables better self-improved performance than random sampling.}
    \label{selfreacher}
\end{figure}

\section{Conclusion}
We studied the prediction–planning gap in latent world models through the action preferences induced by their forward predictors. We showed that alternative actions can receive lower prediction errors than the factual action while producing outcomes farther from the target, indicating that accurate prediction alone does not guarantee reliable planning. To address this, we introduced Action–Consequence Alignment (ACA), which explicitly trains local action comparisons using factual transitions as references. ACA improves planning performance, reduces real goal error, and remains effective with limited data. The same inconsistency signal can further guide self-improving data collection, enabling the model to acquire more informative experience than random local sampling. Together, these results establish action–consequence alignment as a practical principle for both reliable planning and self-improvement.

\clearpage
\bibliographystyle{unsrt}
\bibliography{references}

\clearpage

\appendix

% Place after \appendix and before Appendix A.
\clearpage
\section*{Appendix Contents}

\begingroup
\large
\setlength{\parindent}{0pt}

% Main section entry
\newcommand{\appentry}[3]{%
    \par\addvspace{0.9em}%
    \noindent
    \textbf{#1\quad #2}\nobreak
    \dotfill\nobreak
    \pageref{#3}\par
}

% Subsection entry
\newcommand{\appsubentry}[3]{%
    \par\addvspace{0.45em}%
    \noindent\hspace*{1.5em}%
    #1\quad #2\nobreak
    \dotfill\nobreak
    \pageref{#3}\par
}

\appentry{A}
    {Algorithm}
    {algorithm}

\appentry{B}
    {Local Action Comparisons and Their Physical Consequences}
    {app:ordering_reversals}

\appentry{C}
    {Theoretical Reference for Action Comparisons}
    {theory}

\appentry{D}
    {Datasets}
    {data}

\appentry{E}
    {Experiment Details}
    {experiment details}

\appsubentry{E.1}
    {Architecture}
    {app:architecture}

\appsubentry{E.2}
    {Training and Planning Details}
    {app:training-planning}

\appsubentry{E.3}
    {Hyper-parameter Selections and Sensitive Analysis for ACA}
    {app:hyperparameters}

\appentry{F}
    {Extra Visualization Results}
    {app:visualizations}

\appentry{G}
    {Performance Under Additional Planning Settings}
    {sec:additional_planning}

\appentry{H}
    {Ablation of Action Search and Regularization}
    {app:aca_ablation}

\appentry{I}
    {Variability Across Planner Seeds}
    {appendixsd}
\appentry{J}
    {Discussion}
    {discussion}
\appentry{K}
    {Scope and Limitation}
    {scope}
\endgroup
\clearpage

\section{Algorithm} \label{algorithm}
\begin{algorithm}[t]
\caption{Action--Consequence Alignment (ACA)}
\label{alg:aca}
\begin{algorithmic}[1]
\Require Dataset $\mathcal D$, encoder $f_\theta$,
predictor $P_\phi$, perturbation radius $\rho$,
coefficients $\lambda_{\mathrm{anti}},\lambda_{\mathrm{ACA}}$,
stabilizer $\epsilon$

\For{each minibatch $\mathcal B\sim\mathcal D$}
    \State Encode observations for each
    $\tau=(x_t,a_t,x_{t+1})\in\mathcal B$:
    $z_t=f_\theta(x_t)$, $z_{t+1}=f_\theta(x_{t+1})$

    \State Define prediction error:
    $E_\phi(a;\tau)=\|P_\phi(z_t,a)-z_{t+1}\|_2^2$

    \State Compute action gradient:
    $g_t=\left.\nabla_a E_\phi(a;\tau)\right|_{a=a_t}$

    \State Estimate the most confusable action:
    $\hat a_t=\operatorname{sg}\!\left[
    \Pi_{\mathcal A}\!\left(
    a_t-\rho\frac{g_t}{\|g_t\|_2+\epsilon}
    \right)\right]$

    \State Compute ACA loss:
    $\mathcal L_{\mathrm{ACA}}
    =\left[
    \frac{1}{|\mathcal B|}
    \sum_{\tau\in\mathcal B}
    \left(
    E_\phi(a_t;\tau)-E_\phi(\hat a_t;\tau)
    \right)
    \right]_+$

    \State Compute total objective:
    $\mathcal L
    =\mathcal L_{\mathrm{pred}}
    +\lambda_{\mathrm{anti}}\mathcal L_{\mathrm{anti}}
    +\lambda_{\mathrm{ACA}}\mathcal L_{\mathrm{ACA}}$

    \State Update model parameters by minimizing $\mathcal L$
\EndFor
\end{algorithmic}
\end{algorithm}
\section{Local Action Comparisons and Their Physical Consequences}
\label{app:ordering_reversals}
% Preserve the original label of Appendix B here, if different.

We examine the local manifestation of the prediction–planning gap through two complementary evaluations. In this section, we measure how frequently a locally searched action receives lower prediction error than the factual action. we then execute lower-error alternatives and measure how far their outcomes deviate from the corresponding factual outcomes. Together, these evaluations connect the predictor relative action preferences with physical attainment of the factual target.

\paragraph{Lower-error alternatives.}
We evaluate frozen trained baseline world models, SMWM and LeWM, on
$N=10000$ valid one-step transitions sampled uniformly from the training data of each environment.
For a transition $\tau=(x_t,a_t,x_{t+1})$, we encode the observations as
$z_t=f_\theta(x_t)$ and $z_{t+1}=f_\theta(x_{t+1})$, keep both
representations fixed, and the transition-wise prediction error can be defined as
\begin{equation}
    E_\phi(a;\tau)
    =
    \left\|P_\phi(z_t,a)-z_{t+1}\right\|_2^2.
    \label{eq:appendix_explanatory_error}
\end{equation}
Starting from the factual action, we construct an alternative using
one normalized gradient-descent step:
\begin{equation}
    a_{\mathrm{cf}}^{(\rho)}
    =
    \Pi_{\mathcal A}
    \left(
        a_t
        -
        \rho
        \frac{
            \left.\nabla_a E_\phi(a;\tau)\right|_{a=a_t}
        }{
            \left\|
                \left.\nabla_a E_\phi(a;\tau)\right|_{a=a_t}
            \right\|_2
        }
    \right),
    \qquad
    \rho\in\{0.05,0.10,0.25\},
    \label{eq:appendix_counterfactual_search}
\end{equation}
where 
$\Pi_{\mathcal A}$ enforces the action bounds. The radius $\rho$ specifies the nominal step length; projection can
reduce the actual displacement.
For each radius, the candidate is constructed independently from the
factual action, rather than through successive updates.

A searched alternative receives lower prediction error when
\begin{equation}
    E_\phi\!\left(a_{\mathrm{cf}}^{(\rho)};\tau\right)
    <
    E_\phi(a_t;\tau).
    \label{eq:appendix_reversal_condition}
\end{equation}
We report the frequency of lower error alternatives:
\begin{equation}
    \widehat p_{\mathrm{lower}}(\rho)
    =
    \frac{1}{N}
    \sum_{i=1}^{N}
    \mathbf{1}
    \left[
        E_\phi\!\left(
            a_{\mathrm{cf},i}^{(\rho)};\tau_i
        \right)
        <
        E_\phi(a_i;\tau_i)
    \right].
    \label{eq:appendix_reversal_rate}
\end{equation}

\begin{figure}[tbp]
    \centering
    % Replace with the existing Figure 9 file path.
    \includegraphics[width=\linewidth]{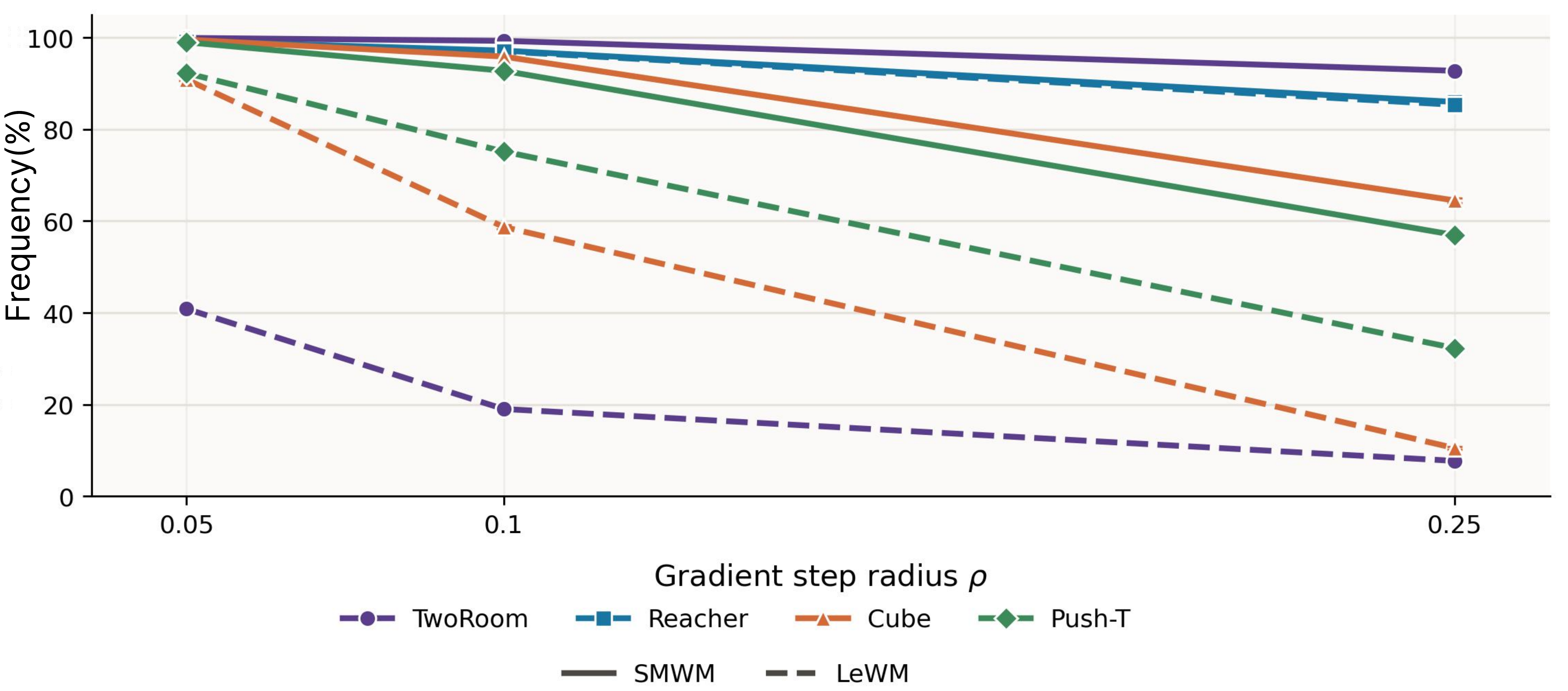}
    \caption{
        \textbf{Frequency of lower-error alternatives.}
        Frozen SMWM and LeWM models are evaluated on 10,000 transitions
        per environment.
        At each nominal perturbation radius, an alternative action is
        constructed using one normalized gradient-descent step from
        the factual action.
        The plotted frequency measures the fraction of transitions for
        which this alternative receives lower prediction error
        against the same observed future.
    }
    \label{fig:ordering_reversal_rates}
    % Preserve the original Figure 9 label here, if different.
\end{figure}

Figure~\ref{fig:ordering_reversal_rates} shows that lower-error
alternatives remain detectable as the nominal perturbation radius
increases from $0.05$ to $0.25$.
Although these frequencies generally decrease with increasing radius,
substantial rates persist in several model-environment settings at
the largest tested radius.
Thus, the observed lower-error alternatives are not confined to the smallest tested step length.
These statistics characterize the action preferences of the predictor;

\paragraph{Physical outcomes of lower-error alternatives.}
We extend the controlled execution example in the main text to
TwoRoom, Reacher, OGBench-Cube, and Push-T, using the same trained frozen baseline models discussed in this section.
For each model, environment, and perturbation radius, we evaluate
1000 transitions whose searched alternatives satisfy
Equation~\ref{eq:appendix_reversal_condition}.
For each evaluated transition, we restore the corresponding complete
simulator state $s_t$ and execute the alternative action over the
same action horizon as the factual transition.

Let $T(s_t,a)$ denote the simulator state reached after executing
action $a$ over this horizon, and let $s_{t+1}$ denote the recorded
factual outcome.
We measure the physical outcome deviation as
\begin{equation}
    D_{\mathrm{phys}}(a;\tau)
    =
    d_{\mathrm{env}}
    \left(
        T(s_t,a),\,s_{t+1}
    \right),
    \label{eq:appendix_physical_deviation}
\end{equation}
where $d_{\mathrm{env}}$ is an environment-specific distance computed
from simulator state variables.
The reported distances are expressed in pixels for TwoRoom and
Push-T, radians for Reacher, and millimeters for OGBench-Cube.
They are computed independently of the learned latent representations.
% Specify the exact state components and distance aggregation used
% by the evaluation code for each environment.

For each model and environment, let $\mathcal V_\rho$ index the
evaluated transitions with lower-error alternatives at radius $\rho$.
The plotted mean is
\begin{equation}
    \overline D_{\mathrm{phys}}(\rho)
    =
    \frac{1}{|\mathcal V_\rho|}
    \sum_{i\in\mathcal V_\rho}
    D_{\mathrm{phys}}
    \left(
        a_{\mathrm{cf},i}^{(\rho)};\tau_i
    \right).
    \label{eq:appendix_mean_physical_deviation}
\end{equation}
This quantity measures the magnitude of outcome deviation among
the evaluated lower-error alternatives, complementing the frequency of lower-error alternatives in Equation~\ref{eq:appendix_reversal_rate}.

\begin{figure}[t]
    \centering
    % Replace with the new physical-outcome figure file path.
    \includegraphics[width=\linewidth]{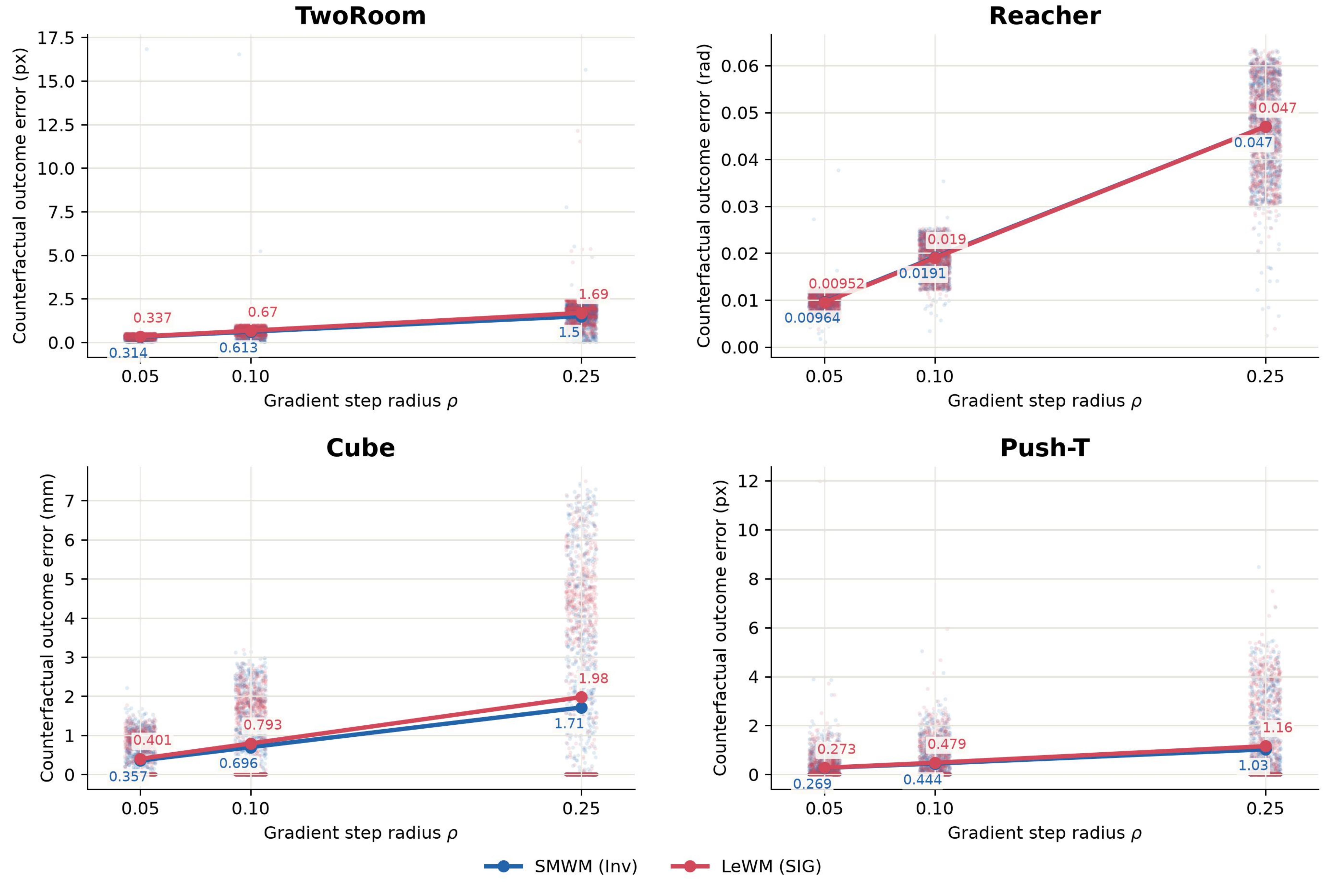}
    \caption{
        \textbf{Physical outcome deviations of lower-error
        counterfactual actions.}
        We execute searched alternatives that receive lower
        prediction error than the factual actions, starting from
        the corresponding recorded simulator states.
        Each model--radius setting contains 1,000 evaluated
        transitions.
        Translucent dots show individual outcome deviations from
        the factual targets, and solid lines connect their means.
        Blue and red denote SMWM and LeWM, respectively.
        Outcome deviations are measured using environment-specific
        simulator state distances, with units indicated on the axes.
        The alternatives receive lower model errors despite
        producing nonzero mean deviations from the factual outcomes.
    }
    \label{fig:counterfactual_outcome_deviation}
\end{figure}

Figure~\ref{fig:counterfactual_outcome_deviation} shows nonzero mean
outcome deviations across all evaluated environments, models, and perturbation radius. The mean deviations increase at larger tested radius, while the
individual transitions exhibit heterogeneous responses.
These measurements show that alternatives favored by the predictor
can produce outcomes that depart from the factual target across
multiple environments.
The relevant discrepancy is the coexistence of lower prediction
error and physical outcome deviation: the predictor assigns a
better score or cost to an alternative that does not reproduce the
observed consequence.

\paragraph{Implications for action selection.}
The recorded future also defines a goal against which candidate
actions can be compared.
When replaying the factual action from $s_t$ reproduces the recorded
transition, $D_{\mathrm{phys}}(a_t;\tau)=0$.
A lower-error alternative with positive physical deviation then
satisfies
\begin{equation}
    E_\phi\!\left(a_{\mathrm{cf}}^{(\rho)};\tau\right)
    <
    E_\phi(a_t;\tau),
    \qquad
    D_{\mathrm{phys}}
    \left(a_{\mathrm{cf}}^{(\rho)};\tau\right)
    >
    D_{\mathrm{phys}}(a_t;\tau).
    \label{eq:appendix_preference_execution_mismatch}
\end{equation}
These inequalities describe a disagreement between the model’s action preference and actual attainment of the factual target. While local descent induced by prediction residuals offers a possible explanation for lower-error alternatives under small perturbations, Figure~\ref{fig:ordering_reversal_rates}  shows that these alternatives persist as the nominal perturbation radius increases from 0.05 to 0.25, with substantial frequencies in several model–environment settings even at the largest tested radius. Figure~\ref{fig:counterfactual_outcome_deviation} further shows that their mean physical outcome deviations increase with the tested radius. Thus, the observed mismatch is not confined to the smallest tested perturbation. Even when such preferences originate from prediction residuals, assigning a lower model error to an action that achieves the target less accurately can misdirect action selection, constituting a concrete manifestation of the prediction–planning gap.

%The two diagnostics characterize complementary aspects of this discrepancy. Figure~\ref{fig:ordering_reversal_rates} measures the frequency of lower-error alternatives under the specified search procedure, while Figure~\ref{fig:counterfactual_outcome_deviation} measures their physical outcome deviations. The planning experiments in the main text further evaluate the practical effect of training with ACA. Together, these results connect explanatory action preferences, their consequences under execution, and the planning benefits of training with factual--counterfactual comparisons.

\section{Theoretical Reference for Action Comparisons}
\label{theory}
For latent action conditioned world model, action selection depends on how prediction error varies across candidate actions. This section characterizes that dependence for an ideal squared-error predictor and examines how conditional variability affects transition-level comparisons, providing a theoretical reference for the local training objective used by ACA.

\paragraph{Setup.}
Fix the encoder $f_\theta$ and a current latent state $z_t$.
For each action under consideration, we assume that the
conditional distribution $p(Z^+\mid z_t,a)$ is well-defined
and has a finite second moment. We then define its conditional mean and the corresponding ideal predictor as
\[
\mu(z_t,a)
=
\mathbb{E}[Z^+\mid z_t,a],
\qquad
P^\star(z_t,a)=\mu(z_t,a).
\]
For a factual action $a_t$, evaluate every candidate action
against the same factual future distribution:
\[
\overline{E}^{\star}(a;z_t,a_t)
=
\mathbb{E}_{Z^+\sim p(\cdot\mid z_t,a_t)}
\left[
\|P^\star(z_t,a)-Z^+\|_2^2
\right].
\]
In this definition, the candidate action is fixed while
the factual future is averaged over.

\begin{theorem}[Expected prediction-error decomposition]
\label{thm:expected_ordering}
Under the above assumptions, every candidate action $a$
satisfies
\[
\overline{E}^{\star}(a;z_t,a_t)
-
\overline{E}^{\star}(a_t;z_t,a_t)
=
\|\mu(z_t,a)-\mu(z_t,a_t)\|_2^2
\geq 0.
\]
Consequently, the factual action is a minimizer of the
expected prediction error. Equality holds if and only if
the two actions have the same conditional mean in the
latent representation.
\end{theorem}

\begin{proof}
Write
\[
Z^+=\mu(z_t,a_t)+\xi,
\qquad
\mathbb{E}[\xi\mid z_t,a_t]=0,
\]
and let
$\Delta(a)=\mu(z_t,a)-\mu(z_t,a_t)$.
Expanding the squared error gives
\begin{align*}
\overline{E}^{\star}(a;z_t,a_t)
&=
\mathbb{E}
\left[
\|\Delta(a)-\xi\|_2^2
\mid z_t,a_t
\right] \\
&=
\|\Delta(a)\|_2^2
+
\mathbb{E}
\left[
\|\xi\|_2^2
\mid z_t,a_t
\right].
\end{align*}
The second term equals
$\overline{E}^{\star}(a_t;z_t,a_t)$,
which proves the result.
\end{proof}

\paragraph{Equivalent actions and non-unique minima.}
The ideal expected error can have multiple minimizing actions. Action redundancy, actuator saturation, and
non-injective dynamics are compatible with the result: actions with the same conditional mean receive the same
expected prediction error. The zero-margin ACA objective also assigns zero penalty to equal errors. Thus, the
comparison does not require recovering a unique action from an observed consequence. Under stochastic transitions, equality of conditional means does not imply equality of
the full transition distributions; the result concerns the conditional means elicited by squared prediction loss.

\begin{corollary}[Prediction-error decomposition for conditionally deterministic transitions]
\label{cor:transition_ordering}
Suppose that the factual latent future is conditionally
deterministic:
\[
Z^+=\mu(z_t,a_t)
\quad\text{almost surely given }(z_t,a_t).
\]
For any resulting transition $\tau$, define
\[
E^\star(a;\tau)
=
\|P^\star(z_t,a)-z_{t+1}\|_2^2.
\]
Then, simultaneously for all candidate actions,
\[
E^\star(a;\tau)-E^\star(a_t;\tau)
=
\|\mu(z_t,a)-\mu(z_t,a_t)\|_2^2
\geq 0.
\]
\end{corollary}

\begin{proof}
Substitute $z_{t+1}=\mu(z_t,a_t)$ into the definition
of $E^\star$.
\end{proof}

The corollary requires that the factual latent future be
determined by the predictor conditioning variables.
A deterministic simulator alone does not establish this
condition: partial observations or the encoder may merge
states with different future outcomes. A sufficient
condition is that the conditioning representation retains
the information needed to determine the next latent state
under the action. Local action injectivity is unnecessary.

\begin{lemma}[Local bound on variability-induced error advantages]
\label{lem:variability-bound}
For each transition $\tau=(x_t,a_t,x_{t+1})$ under consideration,
suppose $\mu(z_t,\cdot)$ is $L$-Lipschitz on
$\mathcal C_\rho(a_t)$, with a common constant $L$.
Let
\[
\xi_\tau=z_{t+1}-\mu(z_t,a_t).
\]
For every $a\in\mathcal C_\rho(a_t)$,
\[
\left[E^\star(a_t;\tau)-E^\star(a;\tau)\right]_+
\le 2L\rho\|\xi_\tau\|_2.
\]
Consequently, for any minibatch $\mathcal B$ and searched actions
$\hat a_t\in\mathcal C_\rho(a_t)$,
\[
\begin{aligned}
\widehat{\mathcal L}_{\mathrm{ACA}}^\star(\mathcal B)
&:=
\left[
\frac{1}{|\mathcal B|}
\sum_{\tau\in\mathcal B}
\left(
E^\star(a_t;\tau)-E^\star(\hat a_t;\tau)
\right)
\right]_+
\\
&\le
\frac{2L\rho}{|\mathcal B|}
\sum_{\tau\in\mathcal B}\|\xi_\tau\|_2.
\end{aligned}
\]
These bounds also hold for alternatives selected using the
observed transitions, provided that each selected action remains
within its corresponding neighborhood.
\end{lemma}

\begin{proof}
For a transition $\tau$ and a candidate action $a$, let
$\Delta_\tau(a)=\mu(z_t,a)-\mu(z_t,a_t)$.
Expanding the squared prediction errors and applying
Cauchy-Schwarz gives
\[
\begin{aligned}
\left[E^\star(a_t;\tau)-E^\star(a;\tau)\right]_+
&=
\left[
2\langle\Delta_\tau(a),\xi_\tau\rangle
-\|\Delta_\tau(a)\|_2^2
\right]_+
\\
&\le
2\|\Delta_\tau(a)\|_2\|\xi_\tau\|_2
\\
&\le
2L\rho\|\xi_\tau\|_2.
\end{aligned}
\]
This inequality holds for every action in the neighborhood,
including an observation-dependent selection.

For the minibatch objective, convexity of the positive-part
function gives
\[
\begin{aligned}
\widehat{\mathcal L}_{\mathrm{ACA}}^\star(\mathcal B)
&=
\left[
\frac{1}{|\mathcal B|}
\sum_{\tau\in\mathcal B}
\left(
E^\star(a_t;\tau)-E^\star(\hat a_t;\tau)
\right)
\right]_+
\\
&\le
\frac{1}{|\mathcal B|}
\sum_{\tau\in\mathcal B}
\left[
E^\star(a_t;\tau)-E^\star(\hat a_t;\tau)
\right]_+
\\
&\le
\frac{2L\rho}{|\mathcal B|}
\sum_{\tau\in\mathcal B}\|\xi_\tau\|_2.
\end{aligned}
\]
Taking expectations over minibatches yields the corresponding
bound on the expected ACA objective.
\end{proof}
ACA applies the positive-part operation after averaging the factual-alternative error differences within each minibatch. The lemma bounds the contribution of conditional variability
to this minibatch comparison penalty. Since the searched alternatives depend on the observed futures, minibatch averaging does not in general eliminate this contribution, even for an
ideal predictor.

\paragraph{What ACA constrains. (gradient analysis)}
ACA constrains the relative prediction errors around an observed
action; it does not directly maximize the sensitivity of the predictor
to actions. This distinction is important. Let
\[
r_t=P_\phi(z_t,a_t)-z_{t+1}, \qquad
J_t=\left.\frac{\partial P_\phi(z_t,a)}
{\partial a}\right|_{a=a_t}.
\]
Then
\[
\nabla_a E_\phi(a_t;\tau)=2J_t^\top r_t.
\]
Therefore, a small factual error gradient does not imply an
action-invariant predictor. In particular, an accurate predictor has
$r_t=0$ and hence $\nabla_aE_\phi(a_t;\tau)=0$ even when
$J_t\neq 0$. Locally,
\[
E_\phi(a_t+\delta;\tau)-E_\phi(a_t;\tau)
=
\delta^\top J_t^\top J_t\delta
+O(\|\delta\|^3),
\]
so an action-sensitive predictor can have a stationary factual action while assigning increasing error to nearby actions.

\paragraph{Summary}
The above results and analysis clarify the principle underlying ACA. Theorem~\ref{thm:expected_ordering} shows that, for an ideal predictor, the expected error increase relative to the factual action equals the squared difference between the corresponding latent conditional means, making the factual action a natural reference for local action comparisons. Corollary~\ref{cor:transition_ordering} shows that this relation holds transition-wise when the conditioning variables determine the next latent state, while Lemma~\ref{lem:variability-bound} bounds the contribution
of conditional variability to both individual error advantages
and the resulting minibatch penalty. The gradient analysis further shows that a stationary prediction error at the factual action is compatible with action-sensitive dynamics. These results motivate a straightforward training principle: use each observed action–consequence pair to supervise not only the factual prediction, but also the model  evaluation of nearby actions. ACA implements this principle as a minibatch-level regularizer alongside forward prediction, providing a local comparison signal for the action selection required by planning.

\section{Datasets} \label{data}

\begin{figure}[H]
    \centering
    \includegraphics[width=1\linewidth]{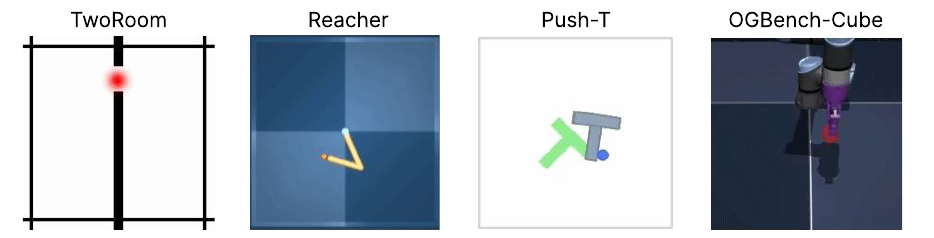}
    \caption{Four evaluation environments}
    \label{environments}
\end{figure}
We conduct experiments on the same four environments (shown in Fig~\ref{environments}) and offline datasets following \cite{maes2026leworldmodel,inverse}. All models are implemented, trained, and evaluated using the codebase of \cite{inverse,maes2026stable}, which provides a unified interface for environment interaction, dataset processing, and downstream planning. For each environment, we split the data at the episode level into 90\% for training and 10\% for validation.The model are optimized only on the training trajectories, whereas all reported result are evaluated exclusively on the held-out validation episodes.

 \textbf{TwoRoom.} \quad  TwoRoom \cite{PLDM,tworomm2} is a two-dimensional point-navigation environment consisting of two rooms connected by a single doorway. The agent begins in one room while the target is located in the other, requiring it to first navigate toward the doorway and then proceed to the goal after crossing the separating wall. The offline dataset contains 10,000 noisy scripted trajectories, with an average episode length of 92 steps.

\textbf{Push-T.} \quad Push-T \cite{pusht} is a planar, contact-rich manipulation task in which a point-like agent must push a T-shaped object toward a desired target configuration. We use the expert demonstrations released with DINO-WM \cite{dinowm}, comprising 20,000 episodes with an average length of 196 steps.

\textbf{OGBench-Cube.} \quad OGBench-Cube \cite{park2025ogbench} is a tabletop manipulation benchmark in which a robotic gripper is tasked with moving a cube to a specified target pose. We consider the single-cube variant and use the benchmark's scripted offline dataset, which consists of 10,000 trajectories of 200 steps each.

\textbf{Reacher.} \quad Reacher \cite{Reacher} is the standard two-link robotic arm environment from the DeepMind Control Suite. Following the DINO-WM formulation, the objective is defined in terms of matching the target joint configuration rather than merely minimizing end-effector distance. We use an offline dataset of 10,000 trajectories generated by a SAC policy \cite{haarnoja2018soft}, with each episode containing 200 steps.

\textbf{OGBench-Cube-Moderate.} The Moderate protocol preserves the original position-based success criterion: an episode succeeds when the Euclidean distance between the cube center and its target is at most 4 cm. Cube orientation and the final robot or gripper configuration are not evaluated, and the first qualifying step counts as success without a persistence requirement. This protocol retains compatibility with the original task criterion while excluding initially solved evaluation pairs.

\textbf{OGBench-Cube-Strict.} The Strict protocol additionally evaluates orientation and short-term persistence. Success requires a cube-center distance of at most 3 cm and an orientation error of at most 15°, simultaneously maintained for three consecutive simulator steps. Orientation error is minimized over the cube 24 rotational symmetries, treating geometrically equivalent orientations as identical. The final robot or gripper configuration remains unconstrained. This protocol evaluates whether the cube reaches and briefly maintains the target pose, rather than merely entering a positional tolerance region.

\textbf{TwoRoom-Strict.} Evaluation pairs are restricted to start and goal states in opposite rooms, with collision-free placements of the full agent disk of radius 7 pixels. The recorded trajectory segment connecting each pair must also be collision-free under continuous swept-disk geometry. During evaluation, collision checking accounts for the full disk throughout each movement, preventing wall crossings that could be missed by endpoint-only checks. Success requires a valid trajectory, at least one legal crossing through the doorway, arrival on the goal side of the wall, and a distance to the goal strictly below 8 pixels. The first step satisfying all conditions counts as success; no additional persistence requirement is imposed.

\section{Experiment Details} \label{experiment details}

\subsection{Architecture} \label{app:architecture}
For the network architecture we used in this paper, we follow the encoder and forward model of LeWorldModel \cite{maes2026leworldmodel} and SMWM \cite{inverse}. Inputs are \(224 \times224\) RGB images. In detail, for the encoder \(f_{\theta}\) we use a Vision Transformer Tiny
(ViT-Tiny) with patch size 14 with approximate
5M parameters. For the forward predictor \(P_{\phi}\) we use a ViT-S backbone (with approximate 10M parameters) where the action \(a\) is injected via AdaLN-zero. We set the history length to 1 on every environment following \cite{inverse}. For SMWM, they use an extra 2-layer
MLP as an inverse model with approximately 0.1M parameters.

\subsection{Training and Planning Details} \label{app:training-planning}
All world models are trained end-to-end under the same settings as LeWM and SMWM. As for planning setting, we use the same MPC and CEM hyperparameters in all four environment and baselines. The detailed hyperparameters for training and planning are shown in Table~\ref{tab:train_plan_hparams}.  
\begin{table}[t]
    \centering
    \caption{Training and planning hyperparameters.}
    \label{tab:train_plan_hparams}
    \renewcommand{\arraystretch}{1.05}
    \setlength{\tabcolsep}{12pt}
    
    \begin{tabular}{lc}
        \toprule
        \textbf{Hyperparameter} & \textbf{Value} \\
        \midrule
        
        \multicolumn{2}{c}{\textbf{Training}} \\
        \midrule
        Optimizer & AdamW \\
        Learning rate & $10^{-4}$ \\
        Weight decay & $10^{-3}$ \\
        Batch size & 256 \\
        Training epochs & 10 \\
        Gradient clipping & 1.0 \\
        
        \midrule
        \multicolumn{2}{c}{\textbf{Planning}} \\
        \midrule
        Planning horizon $H$ & 5 model steps \\
        Executed steps $K$ & 5 model steps \\
        Action block & 5 primitive actions \\
        CEM population & 300 \\
        CEM iterations & 30 \\
        CEM elites & 30 \\
        Initial CEM variance scale & 1.0 \\
        Evaluation tasks & 100 \\
        Goal offset & 25 primitive steps \\
        Evaluation budget & 50 primitive steps \\
        
        \bottomrule
    \end{tabular}
\end{table}

\subsection{Hyper-parameter selections and sensitive analysis for ACA} \label{app:hyperparameters}
We consider two hyperparameters for ACA: the action perturbation
radius $\rho$, which controls the search step for
counterfactual actions, and the loss weight $\lambda_{\mathrm{ACA}}$,
which controls the contribution of the ACA comparison penalty to the overall training objective.
We follow the hyperparameter-selection procedure used in SMWM \cite{inverse} and LeWM \cite{maes2026leworldmodel}, selecting the ACA perturbation radius and regularization weight for each environment and evaluation protocol. 
Importantly, the baseline models are not left untuned: we reproduce SMWM and LeWM using their published, environment-specific recommended regularization weights, while tuning only the two additional hyper-parameters introduced by ACA under the same hyperparameter-selection protocol.
We select $\rho$ from
$\{0.05, 0.1, 0.25, 0.5\}$ and $\lambda_{\mathrm{ACA}}$ from
$\{0.5, 1.0, 1.5\}$. The detailed hyper-parameter of our method is shown in Table~\ref{tab:hyperparameters}. As for the hyper-parameter for SMWM and LeWM, we use the suggested value in their original paper which also shown in Table~\ref{tab:hyperparameters}. We analyze the sensitivity of ACA to the perturbation radius and loss weight on Reacher shown in Figure~\ref{sensitivereacher}. An excessively large perturbation radius weakens the planning gains, highlighting the importance of keeping counterfactual comparisons within an appropriate neighborhood. Meanwhile, ACA consistently outperforms the baseline across the tested nonzero loss weights, suggesting that its benefits do not require precise tuning of the penalty strength within this range.

\begin{figure}[t]
    \centering
    \includegraphics[width=1\linewidth]{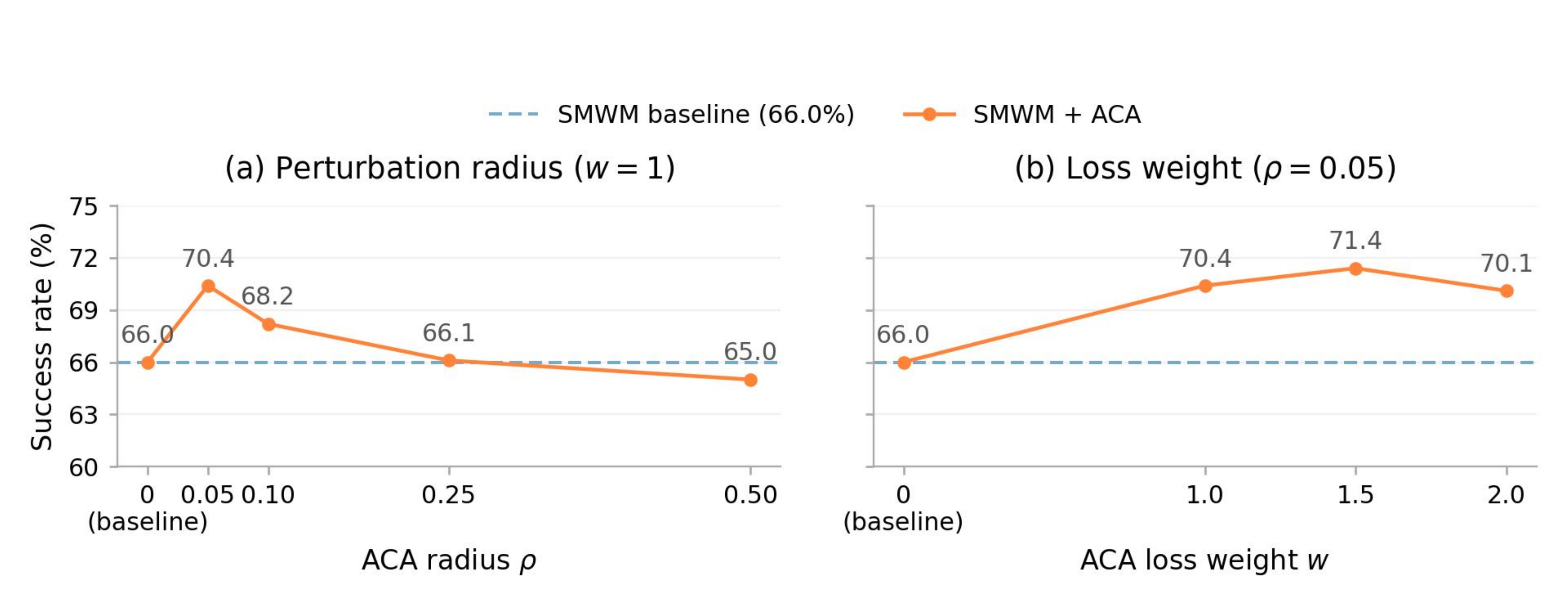}
    \caption{Sensitive Analysis on Reacher}
    \label{sensitivereacher}
\end{figure}
\begin{table}[t]
    \centering
    \caption{Hyperparameter configurations across environments.
    An asterisk denotes the strict evaluation protocol.}
    \label{tab:hyperparameters}
    \small
    \setlength{\tabcolsep}{8pt}
    \begin{tabular}{llcccc}
        \toprule
        Environment & Method
        & InvReg $\lambda$ & SIGReg $\lambda$
        & ACA $\rho$ & ACA $w$ \\
        \midrule
        TwoRoom & SMWM     & 0.1 & --   & --   & --  \\
        TwoRoom & LeWM     & --  & 0.09 & --   & --  \\
        TwoRoom & SMWM+ACA & 0.1 & --   & 0.10 & 1.0 \\
        TwoRoom & LeWM+ACA & --  & 0.09 & 0.10 & 1.0 \\
        \midrule
        TwoRoom$^{*}$ & SMWM     & 0.1 & --   & --   & --  \\
        TwoRoom$^{*}$ & LeWM     & --  & 0.09 & --   & --  \\
        TwoRoom$^{*}$ & SMWM+ACA & 0.1 & --   & 0.25 & 1.0 \\
        TwoRoom$^{*}$ & LeWM+ACA & --  & 0.09 & 0.05 & 1.0 \\
        \midrule
        Reacher & SMWM     & 5.0 & --   & --   & --  \\
        Reacher & LeWM     & --  & 0.09 & --   & --  \\
        Reacher & SMWM+ACA & 5.0 & --   & 0.05 & 1.5 \\
        Reacher & LeWM+ACA & --  & 0.09 & 0.50 & 1.0 \\
        \midrule
        Cube & SMWM     & 1.0 & --   & --   & --   \\
        Cube & LeWM     & --  & 0.09 & --   & --   \\
        Cube & SMWM+ACA & 1.0 & --   & 0.05 & 1.5 \\
        Cube & LeWM+ACA & --  & 0.09 & 0.05 & 1.5 \\
        \midrule
        Cube$^{*}$ & SMWM     & 1.0 & --   & --   & --  \\
        Cube$^{*}$ & LeWM     & --  & 0.09 & --   & --  \\
        Cube$^{*}$ & SMWM+ACA & 1.0 & --   & 0.10 & 1.5 \\
        Cube$^{*}$ & LeWM+ACA & --  & 0.09 & 0.05 & 1.5 \\
        \midrule
        Push-T & SMWM     & 30.0 & --   & --   & --  \\
        Push-T & LeWM     & --   & 0.09 & --   & --  \\
        Push-T & SMWM+ACA & 30.0 & --   & 0.25 & 1.0 \\
        Push-T & LeWM+ACA & --   & 0.09 & 0.25 & 1.5 \\
        \bottomrule
    \end{tabular}
\end{table}

\section{Extra Visualization Results} \label{app:visualizations}

Figure~\ref{extravisualgoalerror} extends the planning-cost analysis with additional examples across environments and base models. Each triplet evaluates the same action sequences on a shared two-dimensional slice, showing the real goal error, baseline planning cost, and ACA planning cost. In these examples, low baseline cost can favor plans with relatively poor physical outcomes, whereas ACA shifts the model preference toward plans with lower real goal error. The relevant improvement is in which plans the model favors, rather than exact agreement between the learned and physical cost landscapes. These visualizations relate changes in learned action preferences to physical goal attainment and provide qualitative evidence that complements the aggregate planning results.

\section{Performance Under Additional Planning Settings}
\label{sec:additional_planning}

We complement the main evaluation with single-plan and
long-horizon settings to examine planning without intermediate
feedback and over longer predicted trajectories.

\paragraph{Single-plan evaluation.}
We select goal observations 25 primitive steps after the initial
observations and set the execution budget to 25 steps.
The planner optimizes five action blocks, each containing five
primitive actions, and executes the resulting sequence without
replanning. This setting tests the quality of the initially
selected plan, since actions cannot be corrected using
intermediate observations.
As shown in Table~\ref{oneplan}, ACA improves the reported
success rates of both base models across all evaluated
environments and protocols, supporting its usefulness when
execution relies entirely on the initial action selection.

\paragraph{Long-horizon evaluation.}
We extend both the goal offset and execution budget to
50 primitive steps. The planner optimizes ten action blocks
at the initial state and executes the full sequence without
replanning. Selecting actions now requires evaluating their
consequences through twice as many autoregressive prediction
steps. This tests whether the local factual--counterfactual
comparisons introduced during one-step training remain useful
for selecting longer action sequences.
Table~\ref{long} shows higher reported success rates with
ACA for both base models across all evaluated environments
and protocols, providing evidence that its planning benefits
extend to longer open-loop execution.

\begin{table}[t]
    \centering
   \caption{Single-plan evaluation.
Success rates (\%) with a goal offset, planning horizon, and
execution budget of 25 primitive steps. A single action
sequence is optimized from the initial observation and
executed without replanning.
An asterisk denotes the strict evaluation protocol. Success rates (\%) are reported as mean, the standard deviation is provided in Table~\ref{tab:planner_seed_sd_25step}.
}
    \label{oneplan}
    \renewcommand{\arraystretch}{1}
    \setlength{\tabcolsep}{12pt}
    
    \resizebox{1\linewidth}{!}{%
    \begin{tabular}{lcccccc}
        \toprule
        \textbf{Method} 
        & \textbf{TwoRoom} 
        & \textbf{TwoRoom*} 
        & \textbf{Reacher} 
        & \textbf{Push-T} 
        & \textbf{\shortstack{OGBench-\\Cube}} 
        & \textbf{\shortstack{OGBench-\\Cube*}} \\
        \midrule
        LeWM
        & $89.8$
        & $37.0$
        & $36.6$
        & $76.6$
        & $60.0$
        & $22.0$ \\

        SMWM
        & $92.6$
        & $25.8$
        & $33.4$
        & $81.6$
        & $82.2$
        & $45.2$ \\
        
        \midrule

        \textbf{LeWM+ACA}
        & $\mathbf{90.2}$
        & $\mathbf{39.6}$
        & $\mathbf{38.4}$
        & $\mathbf{78.6}$
        & $\mathbf{63.0}$
        & $\mathbf{23.8}$ \\

        \textbf{SMWM+ACA}
        & $\mathbf{97.0}$
        & $\mathbf{67.2}$
        & $\mathbf{39.6}$
        & $\mathbf{83.0}$
        & $\mathbf{85.0}$
        & $\mathbf{47.2}$ \\
        
        \bottomrule
    \end{tabular}%
    }
\end{table}

\begin{table}[t]
    \centering
\caption{Long-horizon evaluation.
Success rates (\%) with a goal offset, planning horizon, and
execution budget of 50 primitive steps. The complete
50-step action sequence is planned at the initial state
and executed without intermediate replanning.
An asterisk denotes the strict evaluation protocol. Success rates (\%) are reported as mean, the standard deviation is reported in Table~\ref{tab:planner_seed_sd_additional}}
    \label{long}
    \renewcommand{\arraystretch}{1}
    \setlength{\tabcolsep}{12pt}
    
    \resizebox{1\linewidth}{!}{%
    \begin{tabular}{lcccccc}
        \toprule
        \textbf{Method} 
        & \textbf{TwoRoom} 
        & \textbf{TwoRoom*} 
        & \textbf{Reacher} 
        & \textbf{Push-T} 
        & \textbf{\shortstack{OGBench-\\Cube}} 
        & \textbf{\shortstack{OGBench-\\Cube*}} \\
        \midrule
        
        LeWM
        & 56.0
        & 22.6
        & 26.6
        & 19.2
        & 37.6
        & 16.6 \\
        
        SMWM
        & 55.6
        & 32.0
        & 28.2
        & 33.0
        & 49.0
        & 34.8 \\
        
        \midrule
        
        \textbf{LeWM+ACA}
        & $\mathbf{59.6}$
        & $\mathbf{27.6}$
        & $\mathbf{28.8}$
        & $\mathbf{20.6}$
        & $\mathbf{40.0}$
        & $\mathbf{17.8}$ \\
        
        \textbf{SMWM+ACA}
        & $\mathbf{77.4}$
        & $\mathbf{61.2}$
        & $\mathbf{30.6}$
        & $\mathbf{34.4}$
        & $\mathbf{52.4}$
        & $\mathbf{38.0}$ \\
        
        \bottomrule
    \end{tabular}%
    }
\end{table}

\clearpage
\begin{figure}[p]
    \centering
    \begin{subfigure}{\linewidth}
        \centering
        \includegraphics[
            width=\linewidth,
            height=0.78\textheight,
            keepaspectratio
        ]{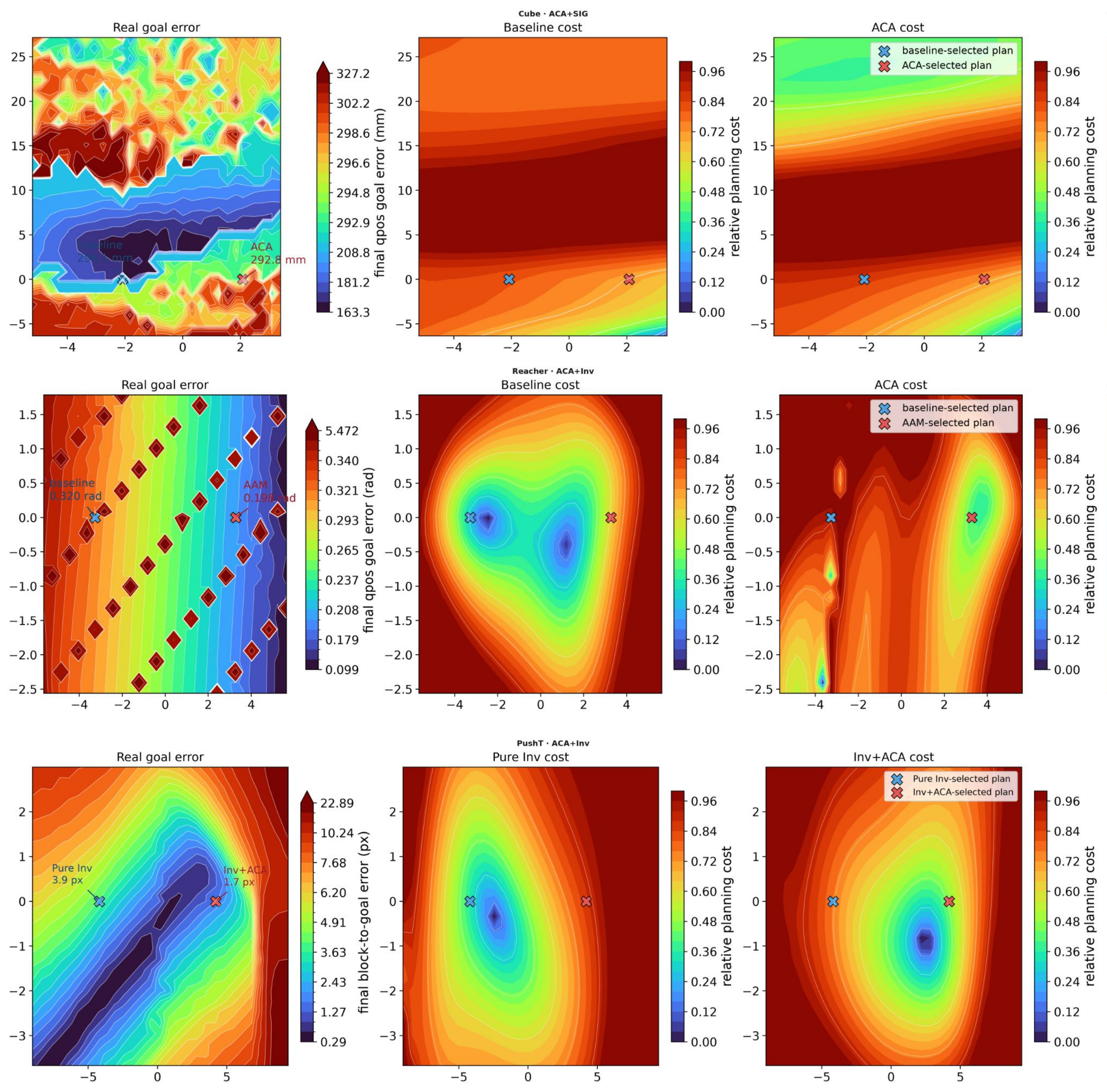}
        \caption{}
        \label{fig:extra-real-error_1}
    \end{subfigure}
    \phantomcaption
\end{figure}

\clearpage
\begin{figure}[p]
    \ContinuedFloat
    \centering
    \begin{subfigure}{\linewidth}
        \centering
        \includegraphics[
            width=\linewidth,
            height=0.78\textheight,
            keepaspectratio
        ]{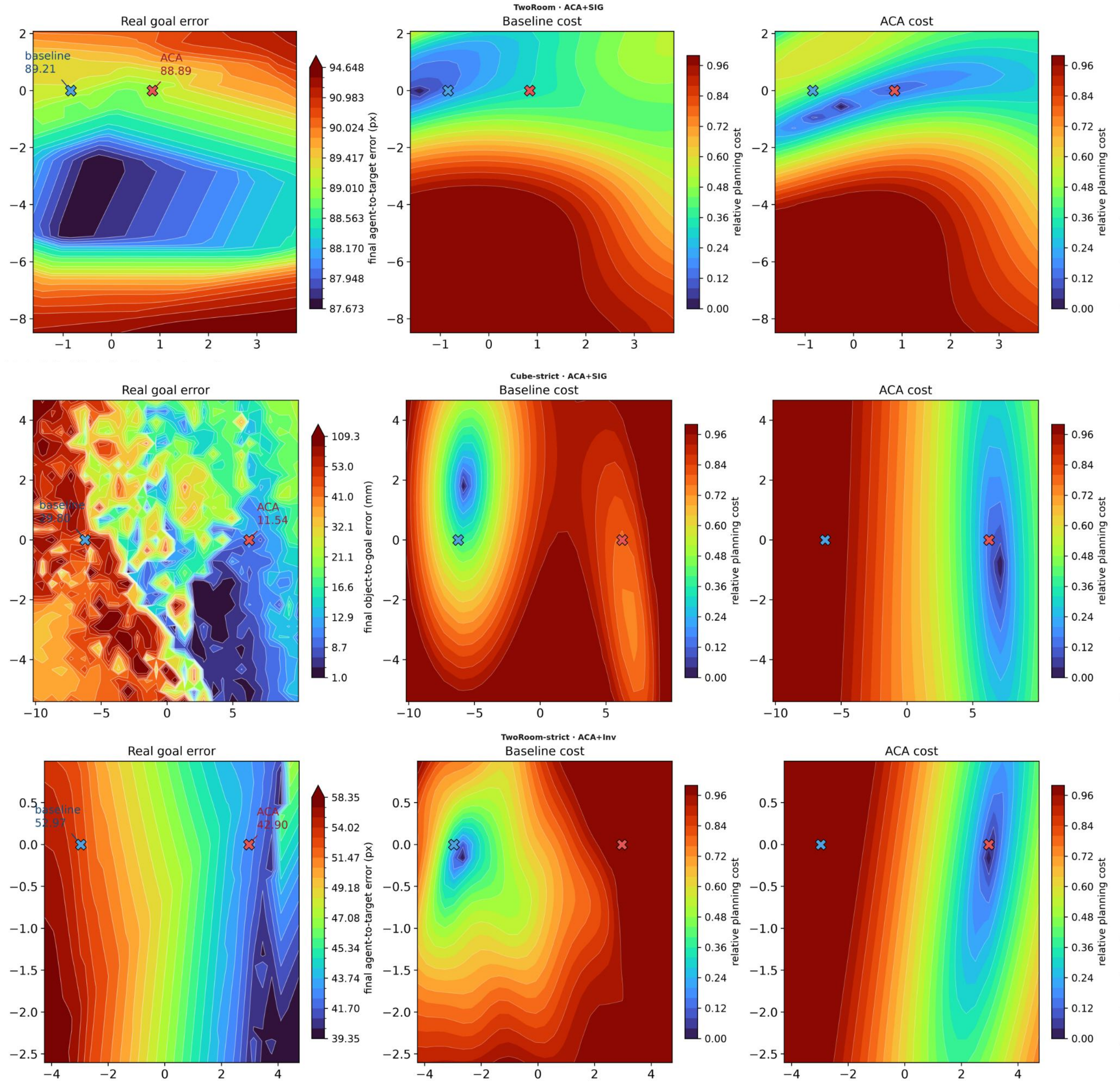}
        \caption{}
        \label{fig:extra-baseline-cost}
    \end{subfigure}
    \phantomcaption
\end{figure}

\clearpage
\begin{figure}[p]
    \ContinuedFloat
    \centering
    \begin{subfigure}{\linewidth}
        \centering
        \includegraphics[
            width=\linewidth,
            height=0.78\textheight,
            keepaspectratio
        ]{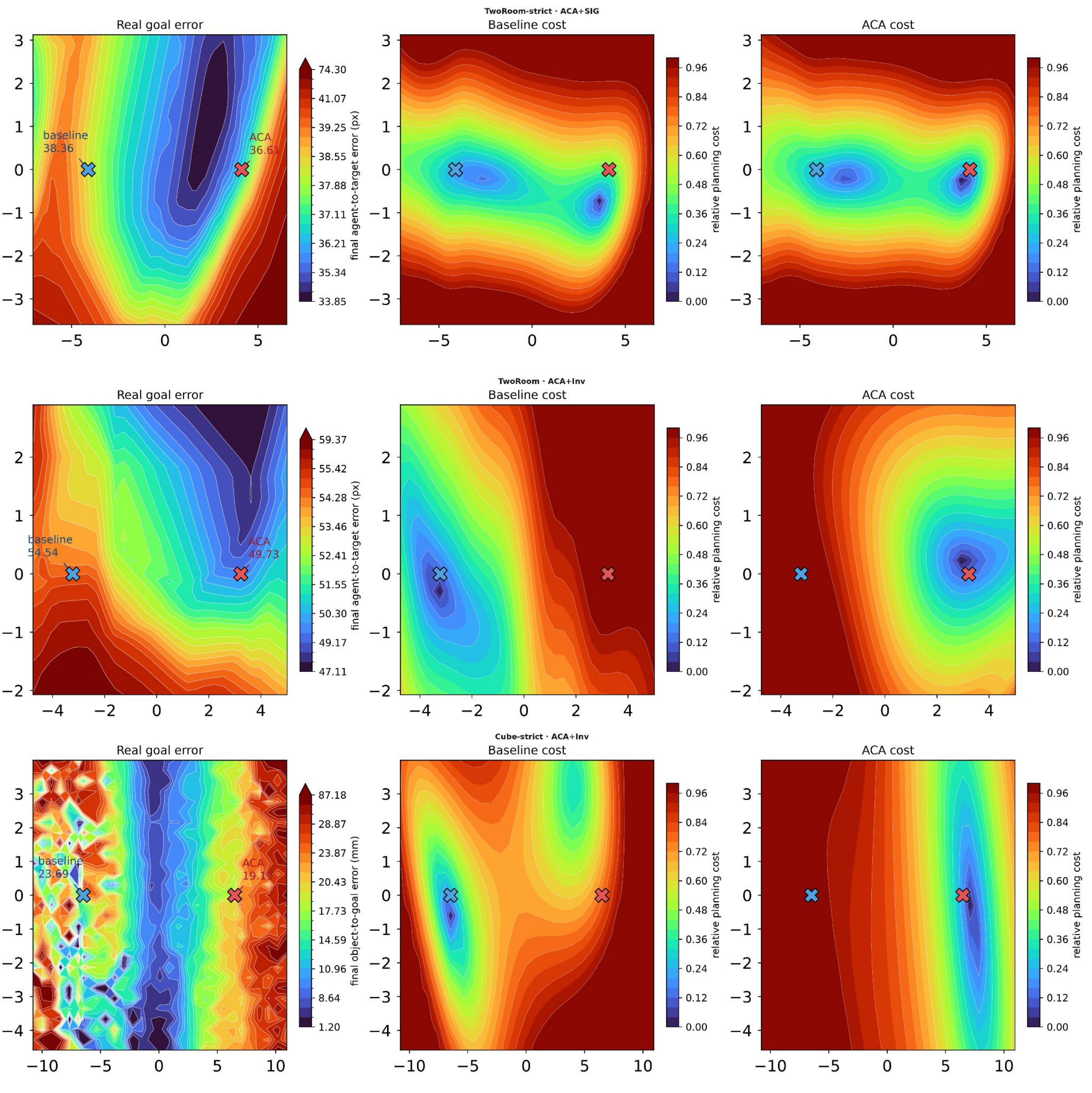}
        \caption{}
        \label{fig:extra-aca-cost}
    \end{subfigure}

    \caption{Additional visualizations of planning costs and physical
    outcomes. Each triplet shows real goal error, baseline planning
    cost, and ACA planning cost for the same action plans on a shared
    two-dimensional slice. The marked plans illustrate how model-based
    action selection translates into physical goal attainment across
    the displayed environments and base models.}
    \label{extravisualgoalerror}
\end{figure}
\clearpage

\begin{table}[t]
    \centering
    \caption{
    Ablation of action search and regularization with SMWM
    trained on 5\% of the training data.
    Success rates (\%) are reported as mean $\pm$ standard deviation
    over five planner seeds.
    An asterisk denotes the CLEAR strict evaluation protocol.
    }
    \label{tab:aca_ablation}
    \small
    \setlength{\tabcolsep}{7pt}
    \renewcommand{\arraystretch}{1.1}
    \begin{tabular}{lccc}
        \toprule
        \textbf{Method}
        & \textbf{Reacher}
        & \textbf{OGBench-Cube}
        & \textbf{OGBench-Cube}$^{*}$ \\
        \midrule
        SMWM
        & $48.0 \pm 3.8$
        & $65.4 \pm 3.3$
        & $24.0 \pm 2.4$ \\
        SMWM + Random-ACA
        & $53.4 \pm 1.7$
        & $66.8 \pm 2.5$
        & $25.6 \pm 4.0$ \\
        SMWM + Gradient-norm penalty
        & $52.0 \pm 6.9$
        & $70.4 \pm 2.5$
        & $31.0 \pm 3.0$ \\
        \midrule
        \textbf{SMWM + ACA}
        & $\mathbf{56.8 \pm 3.7}$
        & $\mathbf{74.0 \pm 2.4}$
        & $\mathbf{36.4 \pm 4.2}$ \\
        \bottomrule
    \end{tabular}
\end{table}
\section{Ablation of Action Search and Regularization}
\label{app:aca_ablation}

We examine how the candidate search direction and the form of
action-based regularization affect planning performance.
Using SMWM as the base model, we compare ACA with random-direction
comparisons and direct action-gradient regularization. Both variants
retain the forward-prediction and inverse-dynamics objectives and
replace only the ACA term.

\paragraph{Comparison objectives.}

\textbf{Random-ACA} replaces the negative-gradient search direction
with a uniformly sampled unit vector while retaining the same
nominal perturbation radius, action-space projection, and
factual--counterfactual comparison:
\begin{equation}
\begin{aligned}
    \hat a_t^{\mathrm{rand}}
    &= \Pi_{\mathcal A}(a_t + \rho u_t),
    \qquad
    u_t \sim \operatorname{Unif}(\mathbb S^{d_a-1}), \\
    \mathcal L_{\mathrm{rand}}
&=
\mathbb E_{\mathcal B,\{u_\tau\}}
\left[
\left[
\frac{1}{|\mathcal B|}
\sum_{\tau\in\mathcal B}
\left(
E_\phi(a_t;\tau)
-
E_\phi(\hat a_t^{\mathrm{rand}};\tau)
\right)
\right]_+
\right].
\end{aligned}
\end{equation}
This variant evaluates the contribution of directing comparisons
toward alternatives with lower prediction error.
\textbf{Gradient-norm penalty} directly regularizes the sensitivity
of the factual prediction error to the conditioning action:
\begin{equation}
    \mathcal L_{\mathrm{grad}}
    = \mathbb E_{\tau}
    \left[\rho \|g_t\|_2\right],
    \qquad
    g_t =
    \left.\nabla_a E_\phi(a;\tau)\right|_{a=a_t}.
\end{equation}
Optimizing this penalty requires differentiating through $g_t$, which introduces mixed second-order derivatives with respect to actions and model parameters. ACA instead stops gradients through the searched candidate and therefore avoids these second-order derivatives through the action-search procedure. For locally smooth errors and small unprojected steps, ignoring
numerical stabilization, the minibatch ACA penalty is approximately
\(
\widehat{\mathcal L}_{\mathrm{ACA}}(\mathcal B)
\approx
\frac{\rho}{|\mathcal B|}
\sum_{\tau\in\mathcal B}\|g_t\|_2.
\) The gradient-norm control captures this leading-order behavior, while ACA evaluates the actual error difference at the perturbed action.

Importantly, this approximation concerns the gradient of the
\emph{prediction error}, rather than the action sensitivity of the
predictor itself. Let
\[
r_t=P_\phi(z_t,a_t)-z_{t+1},
\qquad
J_t=
\left.
\frac{\partial P_\phi(z_t,a)}
{\partial a}
\right|_{a=a_t}.
\]
Then
\[
\nabla_a E_\phi(a_t;\tau)=2J_t^\top r_t.
\]
Hence, reducing the factual error gradient does not in general imply reducing the predictor's action Jacobian. In particular, an accurate action-sensitive predictor may satisfy $r_t=0$ and therefore $\nabla_aE_\phi(a_t;\tau)=0$, while still having $J_t\neq0$. Thus, the first-order approximation above should not be interpreted as showing that ACA necessarily makes the predictor locally action-invariant. At the same time, ACA alone does not exclude an action-independent solution; we therefore regard it as a local
comparison regularizer within the joint training objective.

\paragraph{Experimental setup.}
We use the same 5\% training subset within each environment and
a shared planning protocol across variants. Reacher models are
trained for 20 epochs and OGBench-Cube models for 10 epochs,
using cosine learning-rate annealing from $10^{-4}$.
The regularized variants use
$(\rho,\lambda_{\mathrm{ACA}})=(0.05,1.0)$ on Reacher and
$(0.25,1.5)$ on OGBench-Cube, with $\lambda_{\mathrm{ACA}}$
weighting the corresponding additional objective.
We evaluate Cube under both the standard and CLEAR strict protocols.

\paragraph{Results and analysis.}
Table~\ref{tab:aca_ablation} shows that both controls improve mean planning success over SMWM, suggesting that local action comparisons and error-gradient regularization each provide useful training signals. ACA achieves the highest mean success rate across all three
evaluation settings. Its advantage over Random-ACA supports the value of directing comparisons toward actions that challenge the factual explanation. Relative to the gradient-norm penalty, ACA improves success by 4.8, 3.6, and 5.4 percentage points on Reacher, OGBench-Cube, and OGBench-Cube$^{*}$, respectively. The advantage also holds under the stricter Cube success criterion, supporting its usefulness when goal attainment requires greater
physical precision. Together, these results favor explicit
factual-counterfactual comparisons at a finite perturbation radius over the tested random-direction and local gradient-norm alternatives.

\section{Variability Across Planner Seeds} \label{appendixsd}
Tables~\ref{tab:planner_seed_sd},
\ref{tab:planner_seed_sd_25step},
\ref{tab:planner_seed_sd_additional}
 summarize the standard deviations across five planner random seeds for each fixed trained checkpoint, expressed in percentage points. These results cover the main experiments in this paper. This evaluation follows the benchmark protocol adopted from SMWM; we therefore keep the trained checkpoint fixed and vary the stochastic planner seed for both baseline and ACA models. Across the evaluated environments and protocols, the standard deviations are generally remain comparable in magnitude between the baseline and ACA. These statistics characterize the variability introduced by stochastic planning and provide additional context for interpreting the corresponding mean performance.

\begin{table}[t]
    \centering
    \caption{Standard deviation of success rates across five planner
    seeds under the 25-step, 2-plan evaluation protocol
    (in percentage points). The main results is in Table~\ref{latentplanning}.}
    \label{tab:planner_seed_sd}
    \begin{tabular}{lcccc}
        \toprule
        & \multicolumn{2}{c}{LeWM}
        & \multicolumn{2}{c}{SMWM} \\
        \cmidrule(lr){2-3} \cmidrule(lr){4-5}
        Environment & Baseline & + ACA & Baseline & + ACA \\
        \midrule
        Cube          & 2.88 & 1.73 & 1.64 & 3.35 \\
        Cube-strict   & 3.36 & 2.77 & 3.03 & 2.39 \\
        Reacher       & 3.11   & 2.19 & 3.58 & 2.65 \\
        Push-T        & 2.35 & 3.85 & 3.21 & 2.30 \\
        TwoRoom       & 1.53  & 1.22 & 1.10 & 1.10 \\
        TwoRoom-strict & 2.68 & 3.51 & 2.12 & 2.59 \\
        \bottomrule
    \end{tabular}
\end{table}
\begin{table}[t]
    \centering
    \caption{Sample standard deviation of success rates across five
    planner seeds under the 25-step, 1-plan evaluation protocol
    (in percentage ). The main result is shown in Table~\ref{oneplan}.}
    \label{tab:planner_seed_sd_25step}
    \begin{tabular}{lcccc}
        \toprule
        & \multicolumn{2}{c}{LeWM}
        & \multicolumn{2}{c}{SMWM} \\
        \cmidrule(lr){2-3} \cmidrule(lr){4-5}
        Environment & Baseline & + ACA & Baseline & + ACA \\
        \midrule
        Cube          & 2.55 & 1.58 & 1.64 & 1.87 \\
        Cube-strict   & 3.74 & 2.17 & 3.96 & 4.97 \\
        Reacher       & 2.65 & 2.51 & 2.88 & 2.79 \\
        Push-T        & 2.70 & 2.12 & 1.67 & 3.35 \\
        TwoRoom       & 0.84 & 1.79 & 2.07 & 1.41 \\
        TwoRoom-strict & 1.00 & 2.51 & 5.17 & 4.76 \\
        \bottomrule
    \end{tabular}
\end{table}
\begin{table}[t]
    \centering
    \caption{ Standard deviation of success rates across five planner seeds under the 50-step, 1-plan evaluation protocol (in percentage points). The main result is shown in Table~\ref{long}.}
    \label{tab:planner_seed_sd_additional}
    \begin{tabular}{lcccc}
        \toprule
        & \multicolumn{2}{c}{LeWM}
        & \multicolumn{2}{c}{SMWM} \\
        \cmidrule(lr){2-3} \cmidrule(lr){4-5}
        Environment & Baseline & + ACA & Baseline & + ACA \\
        \midrule
        Cube          & 3.29 & 1.41 & 3.24 & 2.30 \\
        Cube-strict   & 1.87 & 2.24 & 1.79 & 3.54 \\
        Reacher       & 2.82 & 3.12 & 3.96 & 3.63 \\
        Push-T        & 1.14 & 2.39 & 4.69 & 2.30 \\
        TwoRoom       & 3.54 & 4.16 & 1.52 & 0.55 \\
        TwoRoom-strict & 4.16 & 2.30 & 3.83 & 2.78 \\
        \bottomrule
    \end{tabular}
\end{table}

\section{Discussion} \label{discussion}

\textbf{Positioning ACA.} \quad
The approaches discussed in section~\ref{related work} emphasize preserving action-relevant information and making predictions useful for decision making. Our work makes the connection between these requirements explicit through the action evaluations induced by the predictor. For latent action-conditioned world models exemplified by joint-embedding approaches, planning evaluates candidate actions through predicted latent trajectories. The predictor therefore determines not only what future is imagined, but also which action appears preferable. We examine this connection at the level of individual transitions. This makes the prediction-planning gap concrete as a discrepancy between model preference and physical target attainment. It also provides a direct training opportunity: the factual action-consequence pair supplies a reference for supervising the same comparisons that expose the discrepancy. ACA turns this reference into a local objective that penalizes the error advantage assigned to searched alternatives. In this way, ACA complements the approaches above by translating the requirement of decision-relevant prediction into explicit comparison supervision on the existing forward predictor, without additional outcome labels for the alternatives or a separate evaluation module.

\textbf{A shift in generalization requirements.} \quad
In standard supervised prediction, generalization primarily concerns transfer from a finite training sample to unseen examples, with training and deployment typically organized around the same predictive task. Methods such as SAM~\cite{sam} seek to improve the model generalization performance on unseen data distribution within this framework. Action-conditioned world models used for planning extend this requirement: they learn dynamics from transitions that have occurred, but are subsequently queried to imagine futures under candidate actions for a given goal. Actions thus change from conditions supplied by the dataset to variables actively varied and selected during decision making. The learned dynamics must also be reused across new combinations of action sequences and goals. Transfer therefore occurs not only across examples, but also across uses of the learned knowledge: from accounting for recorded experience to reasoning about and selecting possible actions. This requirement persists even when the environment and its physical dynamics remain unchanged, because transitions covered by data-collection behavior need not cover the possibilities that goal-directed planning must compare. Generalization in world models consequently concerns both prediction on unseen transitions and whether those predictions remain useful under action variation, multi-step composition, and changing goals. The prediction-planning gap reflects a possible disconnect between these requirements: generalization on the predictive task does not automatically imply effective use of the learned dynamics for planning.

\section{Scope and limitation} \label{scope}
Our study focuses on latent action-conditioned world models (mainly JEPA-style models empirically) on simulated visual control tasks, other architectures are not discussed and considered here. We believe extending ACA to larger world models and real-world robotic systems is a promising direction for further validation.
\end{document}